\documentclass[preprint,review,3p]{elsarticle}

\usepackage{amssymb}
\usepackage{amsthm,amsmath,amssymb}
\RequirePackage[numbers]{natbib}
\RequirePackage[colorlinks,citecolor=blue,urlcolor=blue]{hyperref}
\usepackage{array}
\usepackage{algorithmic}
\usepackage{color}
\usepackage{times}
\usepackage{graphicx}
\usepackage{booktabs}

\usepackage{caption} 
\usepackage[ruled,vlined,linesnumbered]{algorithm2e}
\newtheorem{theorem}{Theorem}

\newtheorem{lemma}{Lemma}

\newtheorem{assumption}{Assumption}

\def\bs{\beta _*}
\def\gs{\gamma _*}
\def\ts{\theta _*}

\def\Rp{\mathbb{R}^p}
\def\R{\mathbb{R}}

\def\Ex{\mathbb E}

\def\hth{\hat{\theta}}

\begin{document}

\begin{frontmatter}



%\title{Joint empirical risk minimization for biased positive-unlabeled %learning}
%\title{A method for jointly learning class posterior probability and propensity score from %biased, positive-unlabeled data.}
\title{Joint empirical risk minimization for instance-dependent positive-unlabeled data}

\author[label1]{Wojciech Rejchel}
\author[label2,label3]{Paweł Teisseyre \corref{cor1}}
\author[label2,label3]{Jan Mielniczuk}
\cortext[cor1]{Corresponding author: Paweł Teisseyre (teisseyrep@ipipan.waw.pl)}
\affiliation[label1]{organization={Nicolaus Copernicus University},
             city={Toruń},
             country={Poland}
            }

\affiliation[label2]{organization={Polish Academy of Sciences},
            city={Warsaw},
            country={Poland}}
\affiliation[label3]{organization={Warsaw University of Technology},
            city={Warsaw},
            country={Poland}}

\author{}

\begin{abstract}
Learning from positive and unlabeled data (PU learning) is actively researched machine learning task. 
The goal is to train a binary classification model based on a training dataset containing    part of  positives which  are labeled, and unlabeled instances. Unlabeled set includes  remaining part of positives   and all negative observations. 
An important element in PU learning is modeling  of the labeling mechanism, i.e. labels' assignment  to  positive observations.
 Unlike in many prior works, we consider a realistic setting for which probability of  label assignment, i.e. propensity score, is instance-dependent.
 In our approach we investigate minimizer of an empirical counterpart of a joint risk which depends  on both posterior probability of inclusion in a positive class as well as on a propensity score. The non-convex empirical risk is  alternately optimised with respect to parameters of both functions. 
 In the theoretical analysis we 
 establish risk consistency of the minimisers using
  recently derived methods from the theory of empirical processes.
  Besides,
 the important development here is a proposed novel implementation of an optimisation algorithm, for  which   sequential approximation of  a set of positive observations  among unlabeled ones is crucial. This relies on modified technique of 'spies' as well as on  a thresholding rule based on conditional probabilities.
Experiments conducted on 20 data sets for various labeling scenarios show that the proposed method works on par or more effectively than state-of-the-art  methods based on propensity function estimation.
\end{abstract}


% \begin{highlights}
% \end{highlights}

\begin{keyword}
positive-unlabeled learning \sep
propensity score estimation \sep
empirical risk minimization \sep
selected at random assumption\sep
risk consistency



\end{keyword}

\end{frontmatter}


\section{Introduction}
\subsection{PU learning}
Positive-unlabeled (PU) learning is a machine learning task that aims to fit a binary classification model based on  data with partially assigned labels \cite{BekkerDavis2020}. In this scenario,   some instances from a positive class  retain true   labels, while for the remaining instances, no labels have been assigned (they  can belong  either to positive or negative class). PU learning can be thus viewed as a variant of semi-supervised learning \cite{Chapelle2010} when we have positive, negative and unlabeled observations at our disposal; the difference is that in PU case we do not have observations with a negative label assigned. PU data appear in many practical applications. Consider reporting certain ailments, such as migraine attacks, using dedicated mobile applications \cite{Park2016}. Some patients report headaches on days when they occur. However, patients who do not report symptoms may include those who did not experience migraine, as well as those who experienced it but failed to report. Another example is detection of illegal content on social networks. Some profiles are reported as containing illegal content (positive cases). However, profiles not reported as illegal may also contain content that violates the law, but this has not been verified. PU data appear naturally in the classification of texts and images \cite{LiLiu2003,Fung2006,Chiaroni2018}, anomaly detection \cite{Luo2018}, recommendation systems \cite{XMLC} and in many bioinformatics applications \cite{Li2021}, where we often have a small set of positive observations (e.g. confirmed drug-drug interactions) and a large number of observations not assigned to any of the classes.

A key element of PU data analysis is modeling of the labeling mechanism that describes which of the positive observations are assigned a label. This is usually done by imposing some conditions on the  probability of such action, called the propensity score function \cite{BekkerDavis2020}. The simplest approach used for PU data is to assume that this probability is constant (Selected Completely at Random assumption, in short SCAR) \cite{Elkan2001}, which significantly simplifies learning \cite{ElkanNoto2008, Ramaswamy2016, BekkerAAAI18, ICCS2020, TEISSEYRE2021} . However, this assumption is often not met in practice. For example, probability that a patient who has had a migraine attack will report it, may depend on the level of pain but also on other less obvious factors, such as age or 
 her/his handling of migraine reporting application. 
 Similarly, patients who experience certain illness-related symptoms, are more likely to be diagnosed and test positive for the disease, whereas asymptomatic patients may   remain largely undiagnosed.
Much of recent research work, described below,  has focused on a more realistic case when labeling is instance-dependent, in particular commonly adopting
a  less restrictive   assumption, called Selected At Random (SAR), according to which the probability of labeling a positive observation depends solely on the observed feature vector \cite{BekkerRobberechtsDavis2019, Gong2021, Gerych2022, FurmanczykECAI2023, NaVAE, VAEPUCC}. 
%These approaches require the estimation of the propensity function. Most of the developed algorithms consist in alternating estimation of parameters related to the propensity function and the aposteriori probability for the true class variable \cite{BekkerRobberechtsDavis2019, Gong2021, FurmanczykECAI2023} described below.
We also note that PU framework can  be viewed as a special case of data with noisy labels, see \cite{Menon2018, Cannings2020}.
\subsection{Related works}
Presently study of  PU learning refocuses from the scenario when labels are  randomly assigned in a positive class independently from data (SCAR assumption) to the scenario when they may depend on attributes of these elements (SAR assumption). For an overview of PU learning under SCAR, we refer to \cite{BekkerDavis2020}. Recently, some methods have been proposed for PU learning under SAR setting, mainly based on some modifications of loglikelihood methods tailored to the studied  partial observability  scenario. In particular, \textsc{LBE} method \cite{Gong2022} uses Expectation-Maximisation (EM) algorithm and assumption (\ref{double}) below. The Expectation step calculates  conditional probabilities of class given labels and predictors, based on current estimates of parameters, whereas Maximisation step  is based on  the expectation of  conditional likelihood of class indicators given data. SAR-EM \cite{BekkerRobberechtsDavis2019} and TM \cite{FurmanczykECAI2023} methods are based on alternate optimisation of empirical counterparts of separate Fisher consistent criteria for the posterior and the  propensity score. The main difference between those two papers is the way
in  which criterion for propensity score is constructed. Moreover, \cite{FurmanczykECAI2023}  introduced the concept of joint learning of posterior probability and propensity score which extends the method proposed in \cite{ICCS2020} for SCAR framework.
The recent method \cite{Gerych2022} assumes that the propensity score is a linear function of the posterior probability for the true class variable, which is a special case of Probabilistic Gap (PG) Assumption \cite{He2018}. Moreover, the method, called here \textsc{PGLIN}, uses the Positive Function data assumption, stating that for some part of the support of distribution of predictors  posterior probability is equal to $1$. %In experiments, we will refer to the method described in \cite{Gerych2022} as \textsc{PGLIN}.
Finally, in \cite{NaVAE} and \cite{VAEPUCC} deep learning approaches to Empirical Risk Minimisation Method under PU scenario are investigated. In contrast to previously
discussed papers,  parametric modelling of propensity score is avoided, but at the expense of assuming that probability of a positive class is given.  From the theoretical perspective, we also refer to \cite{Coudray2023} where a bound on the expected excess risk is derived under assumption that the propensity function is known.

Finally, it is worth mentioning that there are two basic assumptions for PU data generation: single-sample scenario (SS) and case-control (CC) scenario \cite{BekkerDavis2020}. The SS scenario, assumes that the training data set is an iid sample from a general population being a mixture of positive and negative cases, and the labeled observations are drawn from among the positive ones with a probability described by the propensity score function.
In the CC scenario, the unlabeled data  is drawn from a general population, and the labeled data is an iid sample from the positive class.
In view of this, considering the SAR assumption and the propensity score estimation problem is natural in the case of the SS scenario. On the other hand, most of the work operating for CC, assumes SCAR \cite{duPlessis2014,Kiryo2017,ChenLiuWangZhaoWu2020,Zhao2022,LIU2023, Song2019}. Considering the above discussion, the current work focuses on the SS scenario and SAR assumption and in the experiments we focus on the methods that directly estimate non-constant propensity scores.

\subsection{Contribution and proposed method} The method  proposed  here is based on the observation that the posterior probability for the class label indicator (which indicates whether the observation is assigned a label or not) can be represented  as the product of the posterior probability for the true class variable and the propensity score function. Taking advantage
 of this fact and assuming that the both functions can be modeled parametrically, we propose the optimization of the risk function for the class label indicator, with respect to parameters related to the propensity score function and the posterior probability for the true class variable. 
Due to the fact that both posterior and propensity are modelled as members of the same parametric family,
the described approach suffers from  of lack of identifiability, i.e. the parameters of  propensity score function and the probability for the true class variable can not be specified  uniquely.
However, if certain additional parametric assumptions are imposed on the posterior probability and the propensity score then both  functions are identifiable up to  their interchange. 
We consider  a local minimiser of empirical risk and establish probabilistic bounds on its excess risk (Theorem 1) from which risk consistency of the minimiser defined in (\ref{estim}) below follows (Theorem 2).
%Obviously, in practice, this is still not sufficient
In practice,  we need to effectively find this minimizer, which is not an obvious task due to identifiability issue and non-convexity of the empirical risk.
 To solve the problem, we propose an asymmetric procedure of risk minimisation which differs in the way both parameters are optimised. In each iteration, among the unlabeled observations, we select those that are most likely to be positive. Determining this set is a pivotal problem on which effectiveness of the whole procedure relies and we propose a novel way  to tackle this based on conditional probabilities and technique of spies \cite{BekkerDavis2020,LiuLeeYu2002}. Subsequently, the determined set is used to estimate the propensity score function. Then, given the estimate of the propensity function, we optimize the joint risk function to find  parameters  for the true class variable (see flowchart in Figure \ref{Flowchart}).
 In the experiments, we refer to the proposed method as JERM (\textbf{J}oint \textbf{E}mpirical \textbf{R}isk \textbf{M}inimization).
 
Experiments conducted on 20 data sets, including tabular and image data, and 4 different labeling schemes, including the SCAR and SAR assumptions, show that JERM works comparably to or better than the previously considered methods described above.

Our contributions can be summarized as follows.
\begin{enumerate}
\item  We analyze a joint empirical risk function including parameters for the posterior probability of the true class variable and the propensity score function.
\item We prove an upper bound on the excess risk for the local minimizer of the joint empirical risk function from which its risk consistency  minimiser defined in (\ref{estim})  follows.
\item We propose a new algorithm JERM (\textbf{J}oint \textbf{E}mpirical \textbf{R}isk \textbf{M}inimization), based on the optimization of the joint risk function and the estimation of the propensity score using the spy technique.
\item We design and perform experiments that allow to compare of related methods for different labeling schemes and labeling frequencies.
\end{enumerate}

\section{Positive-unlabeled learning via joint risk optimization}
\subsection{Preliminaries}
We consider a PU setting, where the triple $(X,Y,S)$ is generated from some unknown distribution $P(X,Y,S)$, $X \in \Rp$ is feature vector, $Y \in \{0,1\}$ is true class variable, which is not observed directly and  $S \in \{0,1\}$ is class label indicator. Value $S=1$ indicates that the instance is labeled and thus positive, whereas $S=0$ means that the instance is unlabeled. In PU learning it is assumed that $P(S=1|X=x,Y=0)=0$, which means that negative examples cannot be labeled. 
We adopt single-sample scenario \cite{BekkerDavis2020} in which it is assumed that iid random vectors $(X_i,Y_i,S_i)$ for $i=1,\ldots,n$ are generated from $P(X,Y,S)$. Since $Y_i$ is not observable, the PU training data is $\mathcal{D}=\{(X_i,S_i):i=1\ldots,n\}$.
Observe that in the consider framework $s(x)=P(S=1|X=x)$ can be estimated using PU training data $\mathcal{D}$.
However, our goal is to estimate the posterior for the true class variable $y(x)=P(Y=1|X=x)$. This task cannot be performed directly because we do not observe $Y_i$. 
We note that instance-dependent labeling can be naturally studied in a single-training sample scenario in contrast to case-control scenario in which two samples are available: one iid sample from positive class $P(X|Y=1)$ and one from a general population $P(X)$.

It follows from the Law of Total Probability and assumption $P(S=1|X=x,Y=0)=0$, that the posterior probabilities are related as
\begin{equation}
 \label{posteriorS}  
 s(x)=e(x)y(x),
 \end{equation}
where the propensity score  $e(x)=P(S=1|Y=1, X=x)$ is unknown  and in  a general case  not constant.
In view of (\ref{posteriorS}), identification of posterior $y(x)$ and propensity score $e(x)$  
 is clearly impossible in general. However, if certain parametric assumptions are imposed on $y(x)$ and $e(x)$ then both  functions are identifiable up to an interchange of $y(x)$ and $e(x)$ \cite{FurmanczykECAI2023}. Let $\sigma(t)=1/(1+e^{-t})$ be a logistic function, $a^T$ denotes transposed column vector $a$ and $|a|_1=\sum_{i=1}^p |a_i|$ for $a=(a_1,\ldots,a_p)^T$.
The following parametric assumption will be imposed.
\begin{assumption}
\label{Asm1}
 Posterior probability $y(x)$ and the propensity score function $e(x)$ are described by logistic functions:
\begin{equation}
\label{double}
y(x) = \sigma (\bs ^T x), \quad \quad e(x) = \sigma (\gs ^T x), 
\end{equation}
where $\bs$ and $\gs$ are ground-truth unknown parameters
and $|\bs|_1 > |\gs|_1$.
\end{assumption}

If (\ref{double}) is true, then $y(x)$ and $e(x)$ are identifiable up to their interchange  \cite[Theorem 1]{FurmanczykECAI2023}. Additional condition $|\bs|_1 > |\gs|_1$, ensures that  the functions $y(x)$ and $e(x)$ are identifiable. 
The  role of the  $l_1$ norm played in these conditions  is not essential and it may be replaced by  any norm.
The model (\ref{double}) has been introduced in \cite{Gong2022} and \cite{FurmanczykECAI2023}; in the later reference  it is called double logistic model.
 Obviously, \textbf{Assumption \ref{Asm1}} corresponds to the SAR setting, because the propensity score function depends on the observed features. 
The assumption enables modeling a wide class of situations in which the probability of labeling depends on the feature vector through the sigmoid function. 
Importantly,  we note that \textbf{Assumption \ref{Asm1}}   encompasses situations when  $e(x)$ is not  monotonically increasing function of $y(x)$ and thus
Probabilistic Gap Assumption  considered in \cite{He2018, Gerych2022} is not met.
 %\color{red}
%A1 jest slabsze niz PGA?
\color{black}
The limitation of \textbf{Assumption \ref{Asm1}} is the specific  parametric form of the propensity score function.
%Moreover it is more general than Probabilistic Gap Assumption or assumption made in \textsc{PGLIN} method \cite{Gerych2022}, although its limitation is the parametric form of the propensity function.

Below, we introduce the joint risk function, which will be a core element of our method.
For $a,b \in \R$ and $s \in \{0,1\}$ a  logistic loss function is
\begin{equation}
\label{loss}
\phi(a,b,s)= -s \log [\sigma(a)\sigma(b)] - (1-s) \log [1-\sigma(a)\sigma(b)],
\end{equation}
Let us denote a joint parameter by $\theta = (\beta, \gamma)$ and the ground-truth parameter by $\ts = (\bs, \gs)$. The  risk function is 
\begin{equation}
\label{risk}
Q(\theta)=\Ex \phi(\beta^TX,\gamma^TX,S),
\end{equation}
where  the expectation is taken with respect to both $X$ and $S$.
Moreover, its observable empirical  counterpart (empirical risk)  is 
\begin{equation}
\label{emp_risk}
Q_n(\theta)=\frac{1}{n} \sum_{i=1}^n \phi(\beta^TX_i,\gamma^TX_i,S_i).
\end{equation}
In the proposed approach, we minimize the above function with respect to $\beta$ and $\gamma$ in order to obtain estimators of these parameters.

Function (\ref{emp_risk}) has been already considered in PU learning under SCAR assumption \cite{Lazeckaetal2021} and SAR assumption \cite{FurmanczykECAI2023}. 
It has been shown in  \cite{FurmanczykECAI2023} that under Assumption \ref{Asm1}, a  vector $\ts = (\bs,\gs)$ is the unique minimizer of $Q(\theta)$ over a set $\{\theta=(\beta,\gamma): |\beta|_1 > |\gamma|_1  \}.$  

Table \ref{Tab1} contains the most important notations used in the paper.

\begin{table}
\begin{center}
\caption{Summary of notation.}
\label{Tab1}
\begin{tabular}{ll }
\hline
Notation & Meaning \\
\hline
$n$ & number of instances \\
$p$ & number of features\\
$X\in \R ^p$ & feature vector \\
$Y\in\{0,1\}$ & true class variable (not observed directly) \\
$S\in\{0,1\}$ & label indicator (observed directly) \\
$\mathcal{D}=\{(X_i,S_i):i=1\ldots,n\}$ & PU training data \\
$y(x)=P(Y=1|X=x)$ & posterior probability of $Y=1$ \\
$s(x)=P(S=1|X=x)$ & posterior probability $S$ \\
$e(x)=P(S=1|X=x,Y=1)$ & propensity score function\\
$c=P(S=1|Y=1)$ &  labeling frequency\\
$h(x)=P(Y=1|X=x,S=0)$ & conditional posterior probability of $Y=1$  for unlabeled instance \\
$\bs, \gs$ & ground-truth parameters corresponding to $y(x)$ and $e(x)$\\
$\ts=(\bs,\gs)$ & vector containing all ground-truth parameters \\
$\sigma(t)=[1+\exp(-t)]^{-1}$ & sigmoid (logistic) activation function\\
\hline
\end{tabular}
\end{center}
\end{table}

\subsection{Theoretical analysis for risk minimizers}
In this section we discuss theoretical properties of minimizers of (\ref{emp_risk}).
We  consider the estimator of $\ts$, being a  minimizer of \eqref{emp_risk} in the following sense:
\begin{equation}
\label{estim}
\tilde\theta = \arg \min_{\theta: |\theta|_1 \leq w} Q_n(\theta), 
\end{equation}
where $w>0$ is a radius of a ball around zero, in which we look for the optimal solution. The restriction to a ball with a radius $w$ in \eqref{estim} guarantees the existence of $\tilde \theta$, because $Q_n$ is continuous. Clearly, one wants to take large $w$ in \eqref{estim} and such a choice  will be justified by our theoretical results. 

To be compatible with assumption $|\bs|_1 > |\gs|_1,$ we proceed as follows: first we calculate $|\tilde \beta|_1$ and $|\tilde \gamma|_1$ from \eqref{estim}.  
%for $\hth=(\hat \beta^T, \hat \gamma^T)^T$ .
 If $|\tilde \beta|_1 > |\tilde \gamma|_1,$ then we  let $\tilde \theta = (\tilde \beta, \tilde \gamma), $ otherwise $\tilde \theta = (\tilde \gamma, \tilde \beta). $

Next, we impose some technical assumptions on the distribution of the feature vector that allow us to obtain theoretical results regarding the bounds for excess risk of our estimator.
We assume throughout that components of $X$ are linearly independent almost everywhere (a.e.),  that is $EXX^T$ is positive definite.

\begin{assumption}
\label{Asm2}
We suppose that individual predictors $X_{ij}$'s are sub-Gaussian, i.e. there exists $\mu > 0$ such that for each $j=1,\ldots,p, i=1,\ldots,n$ and  $t \in \R$ we have
$$
\Ex \exp(tX_{ij}) \leq \exp(\mu^2 t^2/2).
$$
\end{assumption}
A family of random variables having sub-Gaussian distributions  is an important generalisation of a family of normally distributed  random variables with mean zero. In this case $\mu^2$ equals the variance  of the corresponding variable. They can be described as variables such that  their  tail can be bounded by a  tail of a certain normal variable $N(0,\sigma^2)$  up to a multiplicative constant. Apart from normally distributed variables the sub-Gaussian  family includes e.g. all bounded random variables (for charaterisation of such variables  see e.g. Theorem 2.6 in \cite{Wainwright2019}).

% Our main  result concerns excess risk, or regret, of $\tilde\theta$ defined as 
% \[R(\tilde \theta)=Q(\tilde \theta) -Q(\theta_*).\]
% Note that $R(\tilde \theta)$ is interpreted as the amount by which the theoretical risk of $Q(\cdot)$  calculated at $\tilde \theta$ for a fixed data set deviates from its minimal value $Q(\theta_*)$(see e.g. \cite{RW2011}).
% \\ 
%To the best of our knowledge this is the first result concerning probabilistic bound of  excess risk for the minimizer of empirical risk in SAR PU setting. 
%\color{red} To zdanie bym wyrzucil, bo w Theorem 10 w Gong et al. (2021) byly nierownosci probabilistyczne dotyczace ryzyka. Caly ten akapit napisalbym tak:
In  \cite{FurmanczykECAI2023}  consistency in estimation of $\tilde \theta$ is established using classical techniques. In the current paper we apply much more sophisticated methods from the theory of empirical processes, which allow us to control
the excess risk, or regret, of $\tilde\theta$ defined as 
\[R(\tilde \theta)=Q(\tilde \theta) -Q(\theta_*).\]
Note that $R(\tilde \theta)$ is interpreted as the amount by which the theoretical risk of $Q(\cdot)$  calculated at $\tilde \theta$ for a fixed data set deviates from its minimal value $Q(\theta_*)$(see e.g. \cite{RW2011}).  
\color{black}
We note that minimiser $\tilde \theta$ is not necessarily unique and the obtained results will  hold for any $\tilde \theta$ in \eqref{estim}.

In Theorem \ref{predict} we investigate properties of a minimiser of \eqref{emp_risk} over a neighborhood of the true parameter $\ts$
\begin{equation}
\label{hh}
\hat \theta = \arg \min_{\theta: |\theta - \theta_*|_1 \leq r} Q_n(\theta),    
\end{equation}
where $r>0$ is an arbitrary number.
Obviously, $\hth$ cannot be found in practice, because its calculation requires knowledge  of $\theta_*.$ However,  risk consistency of the minimiser in \eqref{estim} easily follows from  the bound in  Theorem \ref{predict}. It will be  established in Theorem \ref{risk_consist}.

\begin{theorem}
\label{predict}
Suppose Assumptions \ref{Asm1} and \ref{Asm2}  hold. For each $s \in (0,1)$ and any $\hth$ defined in \eqref{hh} we have
\begin{equation}
\label{predict_claim}
P\left( R(\hth) \leq \frac{32\mu r}{ s} \sqrt{\frac{\log p}{n}}\right) \geq 1-s.
\end{equation}
\end{theorem}

\begin{theorem}
\label{risk_consist}
Suppose Assumptions \ref{Asm1} and \ref{Asm2}  hold.
 For any $\tilde \theta$ in \eqref{estim} we have $R(\tilde \theta) \to 0 $ in probability, if $\log p=o(n)$ and $w=C(n/\log p)^{1/2-\eta}$, where $C>0$ is a  constant and $0<\eta<1/2.$ 
 \end{theorem}

 \begin{proof}[Proof of Theorem \ref{risk_consist}]
Fix $\varepsilon>0.$ In Theorem \ref{predict} we take $r=C(n/\log p)^{1/2-\eta '}$ for $\eta ' \in (0,\eta)$ and $s=\frac{32C\mu }{ \varepsilon} \left(\frac{\log p}{n}\right)^{\eta '}.$ Notice that $s \to 0$ for  $n \to \infty,$ because $\log p=o(n).$ Then for any $\hth$ in \eqref{hh} 
\begin{equation}
\label{prob1}
P(R(\hth)>\varepsilon) \leq s,    
\end{equation}
so this probability tends to zero as $n \to \infty.$

Let $K(\tau,r_0):=\{\theta: |\theta-\tau|_{1}\leq r_0\}$ denote a ball with radius $r_0$, centered at $\tau$. For sufficiently large $n$ we have $K(0,w) \subset K(\ts,(n/\log p)^{1/2-\eta'}), $ because $\log p=o(n)$ and $w=C(n/\log p)^{1/2-\eta}$ for $\eta ' <\eta.$ Due to that the excess risk of any minimiser in \eqref{estim} tends in probability to zero as well. 
 \end{proof}

In Theorem \ref{risk_consist} we establish that the excess risk of $\tilde \theta$ in \eqref{estim} tends to zero even if the radius $w$ tends to infinity as $n \to \infty.$ The rate of $w$ depends on a relation between $\log p$ and $n$, but also on the value $\eta \in (0,1/2).$ As can be observed in the proof of Theorem  \ref{risk_consist}, the value $\eta$ is  a balance between a length of the radius $w$ and a rate that probability \eqref{prob1} tends to zero. %The smaller $\eta$ is, the larger $w$ is, but the slower the rate in \eqref{prob1} is. 
 Smaller $\eta$ allows for larger $w$ to be taken in (\ref{estim}) but this in turn slows down the rate of convergence   of the probability in (\ref{prob1}).
%In Theorem \ref{predict} the radius $r$ can tend  to infinity (but not too fast). 

In the proof of Theorem \ref{predict} we will need the following lemma, which is an extended version of the recent result from \cite{Maurer2016}.  The latter is a multivariate version of the Concentration  Principle \cite[Theorem 4.12]{Ledoux1991}. 

\begin{lemma}
\label{maurer}
Let $z_1,\ldots,z_n$ be fixed elements  from some set $\mathcal{Z}.$ Moreover, let $\mathcal{F}$ be a family of $K$-dimensional functions on $\mathcal{Z}.$  Consider Lipschitz functions $h_i:\mathbb{R}^K \rightarrow \R, i=1,\ldots,n$ with a Lipschitz constant $L>0.$ If there exists $\bar f \in \mathcal{F}$ such that $h_i(\bar f(z_i))=0$ for $i=1,\ldots,n,$ then 
\begin{equation}
\label{Maurer_bound}
\Ex \sup_{f \in \mathcal{F}} \left|\sum_{i=1}^n \varepsilon_i h_i(f(z_i))\right| \leq 2 \sqrt{2} L  \Ex \sup_{f \in \mathcal{F}} \sum_{i=1}^n \sum_{k=1}^K \varepsilon_{i,k} f_k(z_i),
\end{equation}
where $f=(f_1,f_2,\ldots,f_K)$ for each $f \in \mathcal{F}$ and $\{\varepsilon_i\}_i, \{\varepsilon_{i,k}\}_{i,k}$ are independent Rademacher sequences. 
\end{lemma}

The proof of Lemma \ref{maurer} can be found in \ref{App2}.

\begin{proof}[Proof of Theorem \ref{predict}]
We take any $\hth$ satisfying \eqref{hh} and $s \in (0,1).$ We also define a maximum deviation of a  empirical process $Q_n(\theta)- Q(\theta)$: 
\begin{equation}
\label{Un}
U_n(r) = \sup_{ |\theta  - \ts|_1 \leq r} \; |J_n(\theta)|,  
\end{equation}
where $J_n(\theta)=Q_n(\theta) - Q_n(\ts) - Q(\theta) + Q(\ts).$ 
%\color{red}
From the definition of $\hth$ we have  $Q_n(\hat \theta) \leq Q_n(\ts)$ and thus
$$0\leq Q(\hth) - Q(\theta_*)= -J_n(\hth) + Q_n(\hat \theta) - Q_n(\ts) \leq U_n(r).
$$
%\color{black}
Therefore, we focus the attention on bounding \eqref{Un}.

We start with Markov's inequality, which gives $P(U_n(r)> z) \leq \Ex U_n(t) /z$ for any $z>0.$ The choice of $z$ will be given later. Consequently, the main part of the proof is to bound $\Ex U_n(r),$ which is done using tools from the empirical process theory. Some of them are well-known, but we also need  to apply some recent methods from \cite{Maurer2016}. 

 The first step is the Symmetrization Lemma \citep[Lemma 2.3.1]{vaart1998}, which implies that 
\begin{equation}
\label{bound1}
\Ex U_n(r)\leq 2 \Ex \sup_{ |\theta  - \ts|_1 \leq r} 
\left|\frac{1}{n} \sum_{i=1}^n \varepsilon_i \left[\phi(\beta^TX_i,\gamma^TX_i,S_i) - \phi(\bs ^TX_i,\gs ^TX_i,S_i)\right] \right|,
\end{equation}
where $\varepsilon_1, \ldots, \varepsilon_n$ is a Rademacher sequence containing independent sign variables, i.e. $P(\varepsilon_i=1)=P(\varepsilon_i=-1)=0.5.$ This sequence is also  independent of vectors  
$(X_i,Y_i,S_i)$. 

In order to  handle  the bound in \eqref{bound1}, one usually applies the Contraction Principle \cite[Theorem 4.12]{Ledoux1991}. 
However, here we need a multivariate version of this result established in Lemma \ref{maurer}. %, which has been recently proven in \cite{Maurer2016}. 
%In fact, we use its extended version established in Lemma \ref{maurer} in the appendix.
In order to apply it, we first fix $(X_i,S_i)$ and  consider randomness only with respect to $\varepsilon_i.$  Next, we apply a two-dimensional version of Lemma \ref{maurer} with $f_{\beta,\gamma}(x) = ((\beta-\bs)^Tx,(\gamma-\gs)^Tx)$ and functions $h_i, i=1,\ldots,n$ defined as
$$
h_i(a,b) = \phi(\bs^T x_i + a, \gs^Tx_i+b,s_i) -\phi(\bs^T x_i, \gs^Tx_i,s_i) ,
$$
where $x_i$ and $s_i$ are fixed values of $X_i$ and $S_i,$ respectively. It is easily checked  that each $h_i$ is Lipschitz with the  Lipschitz constant $\sqrt{2}$, moreover $h_i(f_{\bs,\gs})=0.$ Therefore,  in view of  Lemma \ref{maurer} we can bound \eqref{bound1} by
\begin{equation}
\label{bound3}
(8/n) \;\Ex \sup_{ |\theta  - \ts|_1 \leq r} \;
\sum_{i=1}^n [\varepsilon_{i,1} (\beta-\bs)^TX_i +\varepsilon_{i,2} (\gamma-\gs)^TX_i],
\end{equation}
where $\varepsilon_{i,1},\varepsilon_{i,2}$  are two independent Rademacher sequences defined in Lemma \ref{maurer}. Notice that the expected value in \eqref{bound3} is again considered with respect to $\varepsilon_i$ and $X_i.$
Obviously, we have $$\sum_{i=1}^n \varepsilon_{i,1} (\beta-\bs)^TX_i\leq |\beta-\bs|_1 | \sum_{i=1}^n \varepsilon_{i,1} X_i|_\infty \leq r | \sum_{i=1}^n \varepsilon_{i,1} X_i|_\infty  ,$$ so we can bound 
\eqref{bound3} by $(16r/n) \;\Ex | \sum_{i=1}^n \varepsilon_{i,1} X_i|_\infty.$ Next, we use {\bf Assumption 1}, i.e. the fact that $X_{ij}$ are sub-Gaussian with a parameter $\mu$, which implies that  $\sum_{i=1}^n \varepsilon_{i,1} X_{ij}$ is sub-Gaussian with a parameter $\sqrt{n} \mu.$ Therefore,  using Lemma 2.2 in \cite{Devroye2012} we obtain 
$\Ex | \sum_{i=1}^n \varepsilon_{i,1} X_i|_\infty \leq 2 \mu \sqrt{n\log p} $, which implies that 
$\Ex U_n(r)\leq 32 \mu r \sqrt{\log p /n}.$ Finally, we take $z=\frac{32 \mu r}{s} \sqrt{\frac{\log p}{n}},$
which finishes the proof.

\end{proof}

\subsection{Proposed algorithm: JERM}
Estimator, defined in (\ref{estim})  has desirable theoretical properties, but its use in practice is problematic. This is due to the fact that  in order to ensure uniqueness of the minimiser of the risk we have to impose
additional assumption $|\beta^*|_1>|\gamma^*|_1$, which is impossible to verify in practice, and therefore  we are unable to distinguish between estimators of $y(x)$ and $e(x)$.
To deal with this challenge, we propose a new, asymmetric procedure, called JERM (\textbf{J}oint \textbf{E}mpirical \textbf{R}isk \textbf{M}inimization) for optimizing the risk function described by  (\ref{emp_risk}) for which optimisation steps for the posterior and the propensity are
different and the ensuing estimators are uniquely determined.
The method alternately determines estimators for $y(x)$ and $e(x)$. We propose to estimate $e(x)$ using model fitted on the set $\hat P$ which at each iteration is a new approximation of  the set of positive examples $P=\{i:Y_i=1\}$. Then, given an estimator of $e(x)$, we optimize function (\ref{emp_risk})  with respect to $\beta$, which allows us to determine the estimator of $y(x)$.
The whole procedure is described by Algorithm \ref{alg-prop} and  below we describe the details of the estimation of $e(x)$.

In order to estimate $e(x)$, we consider the theoretical risk function
\begin{equation}
H(\gamma):=- E_{X,S|Y=1} \left[S\log[\sigma(\gamma^{T}X)]+(1-S)\log[1-\sigma(\gamma^{T}X)] \right].
\end{equation}
The advantage of optimizing function $H(\cdot)$  compared to  optimization of (\ref{risk}) is that the $H(\cdot)$  is convex with respect to its argument, which makes it possible to find a unique solution.
The corresponding empirical risk function is  
\begin{equation}
\label{emp_ex}
H_{P}(\gamma) = -\frac{1}{|P|}\sum_{i\in P} \left[S_i\log[\sigma(\gamma^{T}X_i)]+(1-S_i)\log[1-\sigma(\gamma^{T}X_i)] \right].
\end{equation}
%where $P=\{i:Y_i=1\}$ is a set of positive examples.
As  set $P$ is unknown in PU setting,  it has to be estimated. 
The proposed method for estimating the set $P$ uses the  technique of so-called 'spies'. By a spy we mean an observation belonging to unlabeled set 
$U=\{i:S_i=0\}$, which is the nearest neighbor of a certain observation from the labeled set $L=\{i:S_i=1\}$. 
Note that here we deviate form the usual definition of spies which are defined as labeled examples  added 
to the unlabeled dataset \cite{LiuLeeYu2002}. 
Formally, the set of spies is defined as $SP=\{i\in U: X_i=\text{1-NN}(X_j), j\in L\}$. Intuitively, the spy set contains those unlabeled observations that are close to positive observations in a feature space, so they are also  likely to be positive. Note, however, that apart from the set of spies, which will be assigned to $\hat P$, there may be other observations   among unlabeled observations likely to be positive. Therefore, in addition, we consider observations for which the conditional probability $h(x)=P(Y=1|X=x,S=0)$ is larger  than  the value $h(x)$ for at least one spy.
%added
The rationale here is that we consider as plausible positives those elements which are as likely to be positive, in a specific sense, as at least one spy.
We denote this set by $A=\{i\in U\setminus SP: h(X_i)>\min_{j\in SP} h(X_j)\}$.
Probability $h(x)$ is unknown, however it can be expressed using the functions $e(x)$ and $y(x)$. Namely, denoting by $f(x)$ density of $X$, we have
\begin{equation}
\label{Eq_hx}
h(x) = P(Y=1|X=x,S=0) =\frac{f(x)y(x)(1-e(x))}{f(x)(1-s(x))}=\frac{y(x)(1-e(x))}{1-y(x)e(x)}.
\end{equation}
%The justification for formula (\ref{Eq_hx}) can be found in the appendix.
The function $h(x)$ can be estimated by plug-in the estimator $y(x)$ and the estimator $e(x)$ obtained in the previous iteration. In this way we obtain the estimator of set $A$ (line 7 in Algorithm \ref{alg-prop}): $\hat{A}=\{i\in U\setminus SP: \hat{h}(X_i)>\min_{j\in SP} \hat{h}(X_j)\}$.
Finally, the set of positive examples is estimated as $\hat{P}=L\cup SP\cup \hat{A}$.
The whole procedure is shown as a flowchart in Figure \ref{Flowchart}, whereas Figure \ref{Fig:spies} visualises the process of  determination of $\hat{P}$ based on spies.

\begin{algorithm}
\caption{\textbf{J}oint \textbf{E}mpirical \textbf{R}isk \textbf{M}inimization (\textbf{JERM})}\label{alg-prop}
\begin{algorithmic}[1]
\STATE Input: training data $\mathcal{D}=\{(X_i,S_i):i=1\ldots,n\}$
\STATE Determine set of 'spies': $SP=\{i\in U: X_i=\text{1-NN}(X_j), j\in L\}$.
\STATE Initialize $\hat{e}(X_i)$
\REPEAT 
\STATE Solve $\hat{\beta}=\arg\min_{\beta}Q_n(\beta,\hat{\gamma})$, where $Q_n$ is defined in (\ref{emp_risk}).
\STATE Calculate $\hat{y}(X_i)=\sigma(\hat{\beta}^{T}X_i)$, for $i=1,\ldots,n$.
\STATE Calculate
\[
\hat{h}(X_i)=\frac{\hat{y}(X_i)(1-\hat{e}(X_i))}{1-\hat{y}(X_i)\hat{e}(X_i)},\quad \hat{A}=\{i\in U\setminus SP: \hat{h}(X_i)>\min_{j\in SP} \hat{h}(X_j)\}
\]
\STATE Determine set $\hat{P}=L\cup SP\cup \hat{A}$.
\STATE Solve $\hat{\gamma}=\arg\min_{\gamma} H_{\hat{P}}(\gamma)$, where $H_P(\gamma)$ is defined in (\ref{emp_ex}).
\STATE Calculate $\hat{e}(X_i)=\sigma(\hat{\gamma}^{T}X_i)$, for $i=1,\ldots,n$.
\UNTIL{convergence}
\end{algorithmic}
\end{algorithm}

\begin{figure}[ht!]
\includegraphics[width=0.9\textwidth]{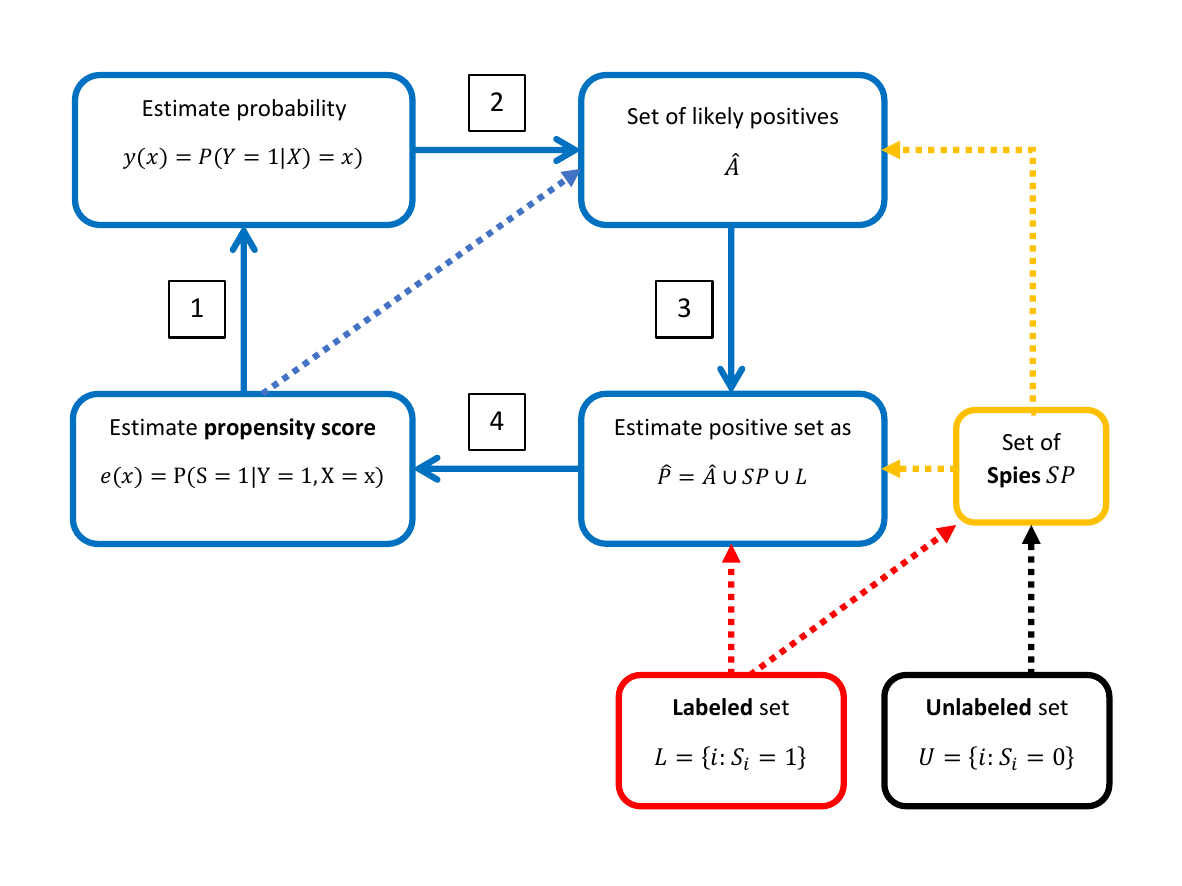} 
\caption{Flowchart of the proposed method JERM. (1) Given an estimator $e(x)$, we find $y(x)$ (line 5 in Algorithm 1). (2) Based on the estimators $y(x)$ and $e(x)$ and the set of spies, we determine the set of observations that are likely positive: $\hat{A}$. (3) The set of positive observations $P$ is approximated by  the sum of the sets $\hat{A}$, the set of spies and the set of positive observations. (4)  $\hat{P}$ is used for estimating $e(x)$ (line 9 in Algorithm 1).}
\label{Flowchart}
\end{figure}

\begin{figure}[ht!]
\includegraphics[width=0.9\textwidth]{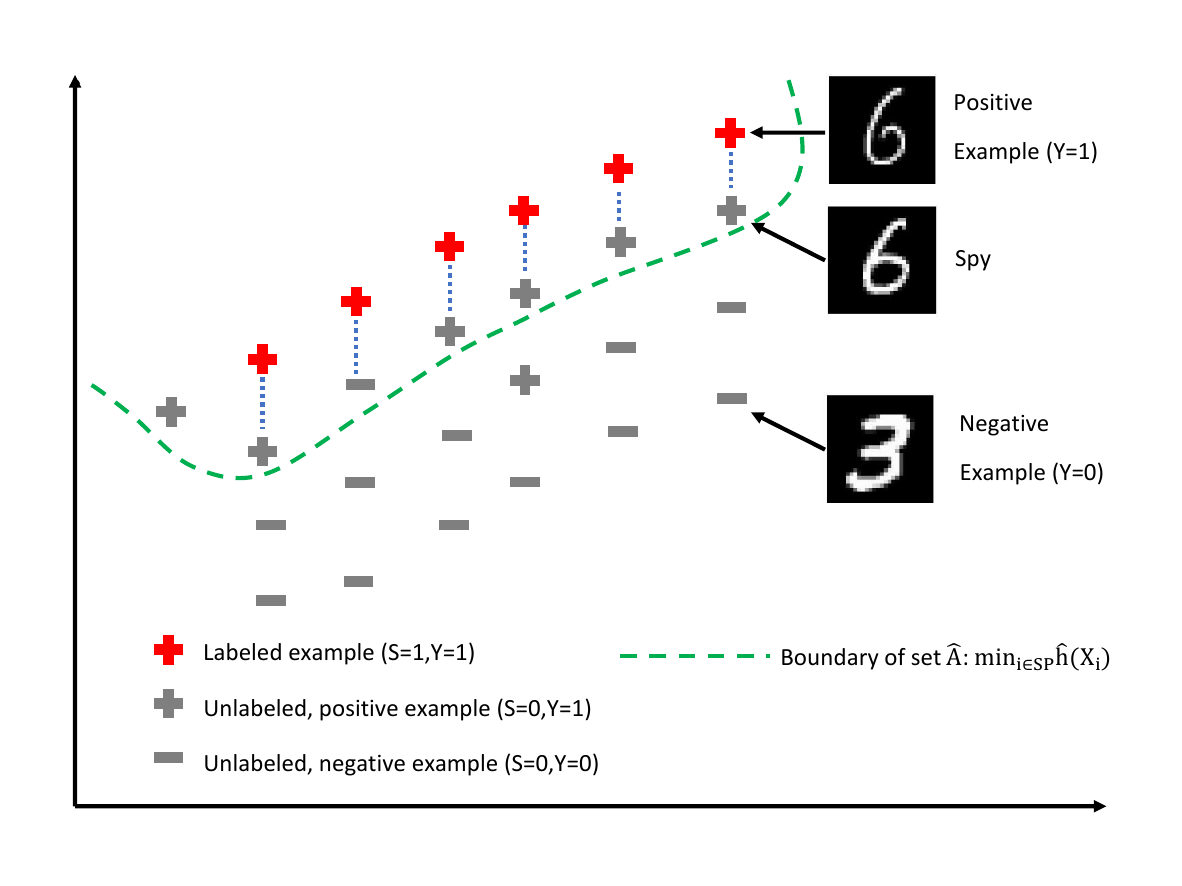} 
\caption{Visualization showing how to determine the set $\hat{P}$ based on spies.}
\label{Fig:spies}
\end{figure}

The above algorithm contains some technical details that call for comments. First, as $Q_n(\beta,\gamma)$ is not convex in either $\beta$ or $\gamma$,  Majorisation-Minimisation (MM) algorithm (see e.g. \cite{Hastie2015}, Section 5.8), commonly employed in such scenarios, is used to find the maximiser in line 5 in Algorithm \ref{alg-prop}. The analogous idea was used in previous papers on PU learning \cite{Lazeckaetal2021, FurmanczykECAI2023}, where it was shown that the use of the MM algorithm is superior to   standard gradient-based algorithms. Secondly, we need to initialize the $e(x)$ estimator (line 3 in Algorithm \ref{alg-prop}). We use a simple estimator $\hat{e}(x)=0.5(1+\hat{s}(x))$, where $\hat{s}(x)$ is estimated by fitting the naive model in which variable $S$ is treated as a class variable. The form of the estimator results from a simple inequality $s(x)= e(x)y(x)\leq e(x)\leq 1$ and considering the average value between $s(x)$ and $1$.

\section{Experiments}
\subsection{Methods and datasets}
We use the most related methods discussed in the previous sections: \textsc{PGlin} \cite{Gerych2022}, 
\textsc{LBE} \cite{Gong2021},
\textsc{SAR-EM} \cite{BekkerRobberechtsDavis2019} and
\textsc{TM} \cite{FurmanczykECAI2023}.
Such methods were chosen because we focus on methods operating under the SAR assumption and those that directly estimate $e(x)$. To make the comparison fair, in all methods the base classifier is a logistic regression (LR). The LR model was also used as the base model in \cite{Gong2021}, \cite{BekkerRobberechtsDavis2019} and \cite{FurmanczykECAI2023}.
The reference method is the \textsc{NAIVE} method, which uses the variable S as the unknown class  variable, i.e. it treats unlabeled observations as negative.
We also consider the \textsc{ORACLE} method, which assumes knowledge of the true variable $Y$. This method cannot be used in practice for PU data, but is useful in experiments because its result can be interpreted as yielding  an upper bound on accuracy.

The experiments were conducted on 20 datasets, including 16 tabular datasets from the UCI repository and 4 image datasets: CIFAR10, MNIST, Fashion and USPS. The characteristics of the sets are provided in Table \ref{Tab:datasets} in Appendix A.
Tabular datasets with multiple classes were transformed into binary classification datasets, such that the positive class includes the most common class, and the remaining classes are combined into the negative class. In the case of image datasets, we define the binary class variable depending on the particular dataset. In CIFAR10, all vehicles form a positive class and animals form a negative class. For MNIST, even and odd digits form the positive and negative classes, respectively. In the USPS positive  digits less than five constitute the positive class, and remaining digits constitute the negative class. In FashionMNIST, clothing items worn on the upper body are marked as positive cases, the remaining images are in the negative class.
For image data, we use a pre-trained  deep neural network Resnet18 to extract the feature vector. For each image, the feature vector of dimension $512$ , is an outocome  of the average pooling layer. Then, the extracted feature vector is used by PU models.
\subsection{Labelling strategies and evaluation methods}
Given a set related to a binary classification problem, we artificially generate PU data using various labeling strategies. We assign all negative observations to the unlabeled set. From  the positive observations we randomly select those that will be labeled with probability $e(x)=P(S=1|X=x,Y=1)$. We are considering the following strategies.
\begin{itemize}
\item [{\bf S1.}] Propensity score  $e(x)=c$. 
\item [{\bf S2.}] Propensity score  $e(x)=F_{\text{Logistic}}(x^{T}\beta^*+a)$.
\item [{\bf S3.}] Propensity score  $e(x)=F_{\text{Cauchy}}(x^{T}\beta^*+a)$. 
\item [{\bf S4.}] Propensity score $e(x)=[F_{\text{Logistic}}(x^{T}\beta^*+a)]^{10}$.
\end{itemize}
In the above formulas, $F_{\text{Logistic}}$ and $F_{\text{Cauchy}}$ denote the cumulative distribution functions of the logistic and Cauchy distributions, respectively.
Parameter $a$ is determined to control the value of labelling frequency $c=P(S=1|Y=1)$ and $\beta^*$ is computed using \textsc{ORACLE} method with logistic regression.
Strategy S1 corresponds to SCAR, whereas strategies S2-S4 correspond to SAR assumption.
Moreover, note that S2 is related to assumption (\ref{double}). The analysis of strategies S3 and S4 allows us to examine the robustness of the method to deviation from the assumption (\ref{double}). The S4 strategy was considered in \cite{Gong2021}.  Note that in this case  propensity score approximates 0-1 step function.

In the case of tabular datasets, we randomly split the data into a training set (75\% of observations) and a test set (25\% of observations). For image datasets, we use predefined splits provided in the PyTorch library \cite{PyTorch19}. From the training data, we generate PU data using the labeling strategies described above. Then, the  PU models are fitted on  training data. Finally, the models are evaluated on test data. We use balanced accuracy as the primary evaluation metric because this metric takes into account class imbalances that occur in some datasets. The above steps are repeated for 10 random data splits and finally we report the average values and standard deviations.

\subsection{Discussion}
The aim of the experiments was to address the following questions. (1) How do the considered  methods work for different labeling schemes? (2) Is the proposed method \textsc{JERM}  robust to deviations from model assumptions? (3) What is the effectiveness of the methods for different label frequencies?
\subsubsection{Analysis of performance for different labeling schemes}
Table \ref{tab:general} shows summary results of pairwise comparisons: wins (W), losses (L) and draws (D) of the proposed  method \textsc{JERM} against each competitive method in terms of
average balanced accuracy, whereas Tables \ref{tab:S1c03}-\ref{tab:S4c07} in Appendix show detailed results for S1-S4 and labeling frequency $c=0.3,0.5,0.7$.

When compared with the \textsc{NAIVE} method and \textsc{SAR-EM}, the  proposed method \textsc{JERM} is the clear winner for most datasets, regardless of labeling scheme and label frequency. Comparison with \textsc{LBE} shows, that for most datasets, the classification accuracy is comparable, however in the case of S4, the proposed method works significantly better for 3-5 datasets, while the opposite situation does not happen.  Note that the labeling  strategy (S4) originates from \cite{Gong2021}, where \textsc{LBE} was proposed.
%\color{red}
%We want to stress that the strategy (S4) is not proposed by us, but is borrowed from \cite{Gong2021}, where \textsc{LBE} was proposed.
\color{black}
Comparison with \textsc{PGLin} reveals that for the strategies S2-S4, the \textsc{JERM} method has significantly higher accuracy for a larger number of datasets. The effectiveness of the \textsc{TM} method depends on the labeling scheme. For S1, \textsc{TM} performs similarly (12-14 datasets) or better (5-7 datasets) than the \textsc{JERM} method. The situation changes for non-SCAR schemes (S2-S4), where for most datasets the \textsc{JERM}  works better or comparable to \textsc{TM}. Importantly, for S3 and S4, on no dataset \textsc{TM}  performs   better than the \textsc{JERM} method.

In addition, Table \ref{tab:general_simple} shows percentage of absolute wins (averaged over $c$) of the \textsc{JERM} method
against each competitive method. The proposed method \textsc{JERM} turns out to be a winner for most datasets and S2-S4 scenarios. In the case of S1, the \textsc{JERM} works better that the \textsc{NAIVE} and \textsc{EM} for most datasets.

The analysis of the average ranks for individual methods (last row in Tables 2-13) indicates that the \textsc{JERM} works on average most effectively for non-SCAR schemes and lower label frequencies ($c=0.3$ and $c=0.5$), while for $c=0.7$ is the second best. For the SCAR scheme S1, the \textsc{TM} method emerges as the winner.

\subsubsection{Robustness to deviations from model assumptions}

Note that our method, as well as \textsc{LBE}, is based on the assumption (\ref{double}) indicating that the propensity function depends on a linear combination of features through a sigmoidal activation function. The labeling schemes S1 and S2 satisfy  this assumption, in contrast to  S3 and  S4 scenarios. In particular, for S4, the propensity function approximates 0-1 step function and deviates significantly from the assumed sigmoid function. Importantly, the \textsc{JERM} achieves the highest averaged ranks for S3 and S4 and lower label frequencies ($c=0.3,0.5$). This points to the robustness of our method when  the propensity score does not follow logistic model. Additionally, for the S4 scheme, the \textsc{JERM}  performs significantly better (3-5 datasets) or comparable (15-17 datasets) to \textsc{LBE}, which may indicate that the proposed method is more resistant to deviation from assumptions than \textsc{LBE}.

\subsubsection{Performance for different labeling frequencies}
The analysis of the situation of low label frequency $c$ is particularly interesting, because in this situation, learning is usually based on small (or even extremely small) number of labelled instances.
Experiments show that for low label frequency, the proposed method is more effective than other competitors. The \textsc{JERM} method turns out to be the winner (in terms of averaged ranks) for S2-S4 and is the second best for S1. This is confirmed by more detailed analyzes carried out on selected datasets (Segment, Banknote, CIFAR10, Fashion), for which we studied the dependence between balanced accuracy and the label frequency $c$ (Figures \ref{Res_cplots1} and \ref{Res_cplots2}). We observe high accuracy of the proposed method even for small label frequency, such as $c = 0.1$. In the case of some datasets and scenarios (e.g. the Segment dataset and the S4 scheme), the advantage is very pronounced. As expected, the performance of all methods increases with $c$, approaching the effectiveness of the \textsc{ORACLE} method. For image datasets, the shapes of the accuracy curves corresponding to the \textsc{JERM} method and \textsc{LBE} are similar, which is understandable because both methods are based on a the same assumption (\ref{double}). Importantly, however, the accuracy curve for the proposed method usually dominates the accuracy curve for LBE.

% Finally, let us mention that the computational costs are comparable for JERM and the most related LBE method. For example, for CIFAR10 data, scheme S1 and $c=0.5$, the computation time for LBE is 133.01 secs, while the computation time for the proposed method is slightly less 120.23. In turn, for MNIST data with the same settings, the computation time for \textsc{LBE} is 164.15 and is slightly smaller than the computation time for JERM which is 178.00 secs.

\subsubsection{Computational issues}
 Finally, we  briefly discuss computational issues. All considered methods based on alternating model fitting for a posterior probability and propensity score function   (\textsc{EM, LBE, TM})  are quite computationally expensive. In the case of the proposed method, the computational cost is additionally increased by the use of the MM algorithm in optimization and determination of spies. However, despite this, the computation times are comparable to the most related \textsc{LBE} method. For example, for CIFAR10 data and scheme S1, the computation time for LBE is 133.01 secs, while the computation time for the proposed method is slightly less 120.23. In turn, for MNIST data and scheme S1, the computation time for \textsc{LBE} is 164.15 and is slightly smaller than the computation time for JERM which is 178.00 secs.

\begin{figure}[ht!]
\centering
$\begin{array}{cc}
\includegraphics[width=0.45\textwidth]{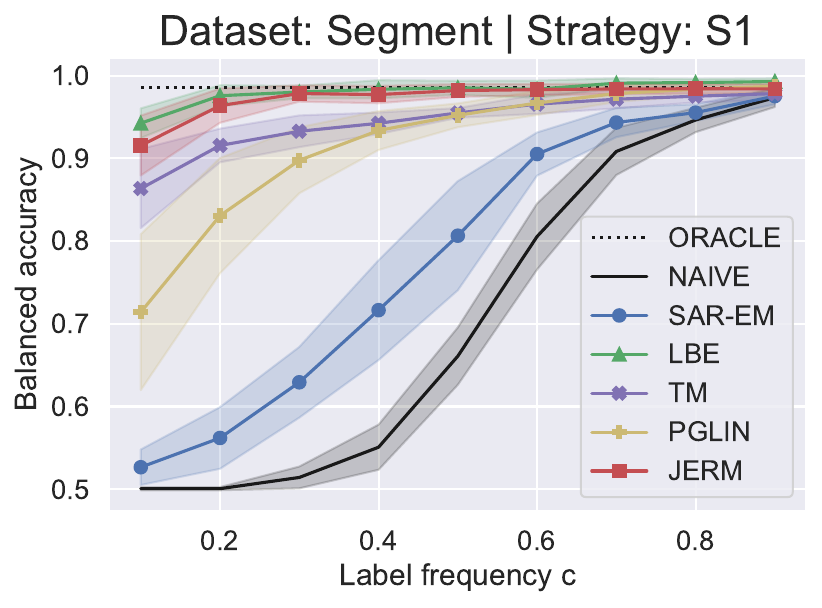} &
\includegraphics[width=0.45\textwidth]{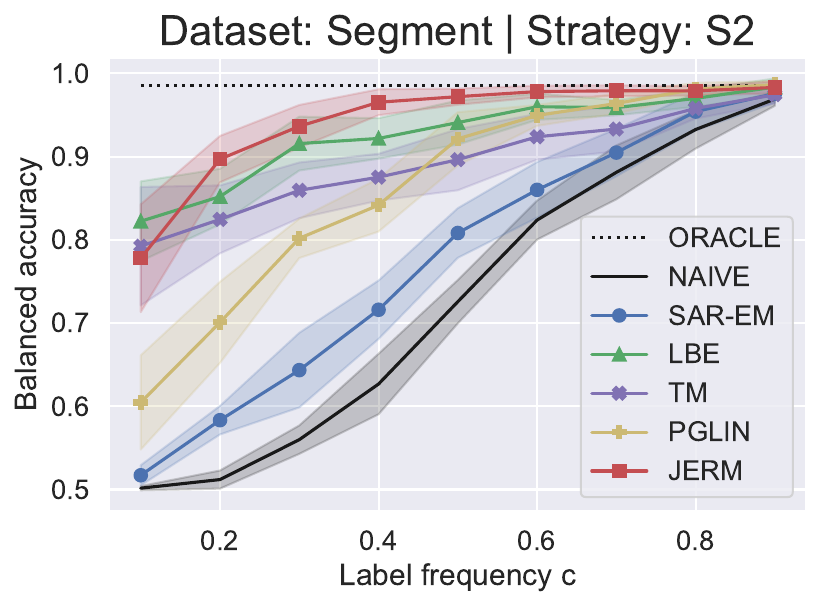} \\
\includegraphics[width=0.45\textwidth]{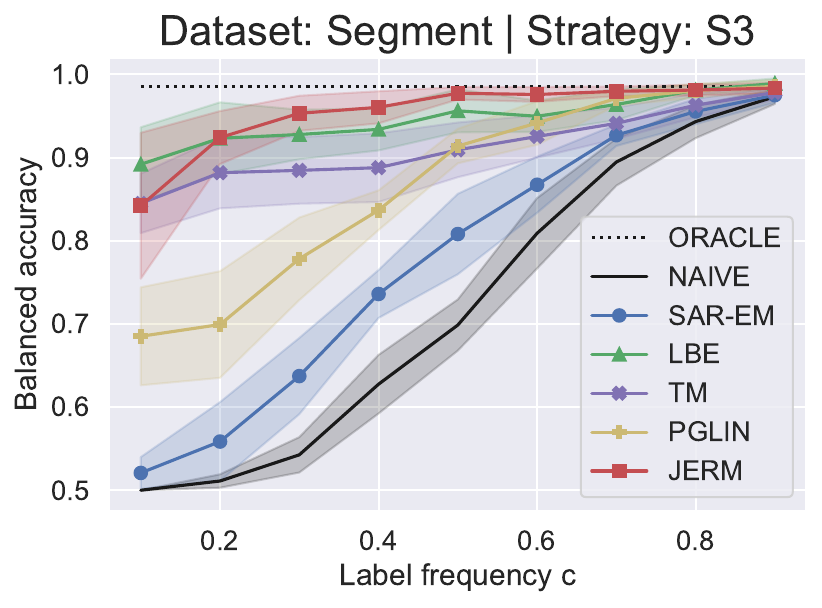}&
\includegraphics[width=0.45\textwidth]{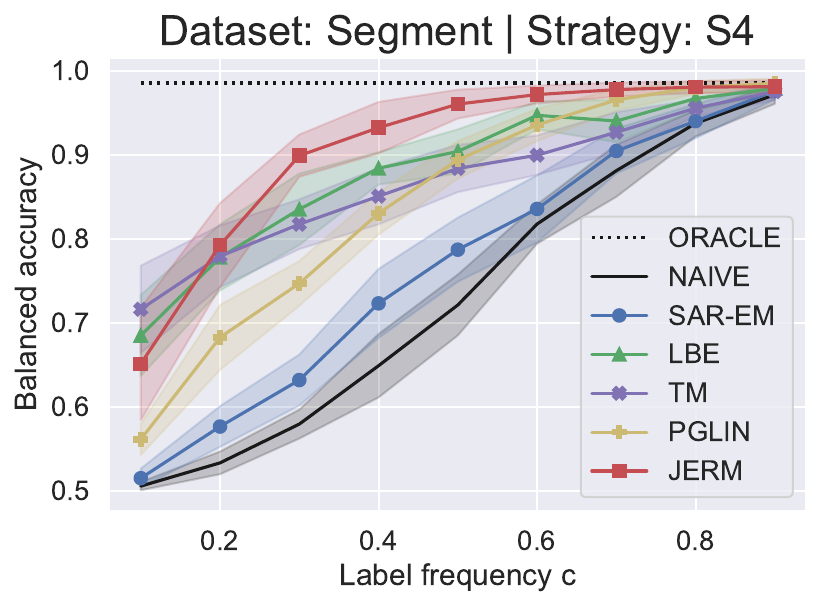} \\
\includegraphics[width=0.45\textwidth]{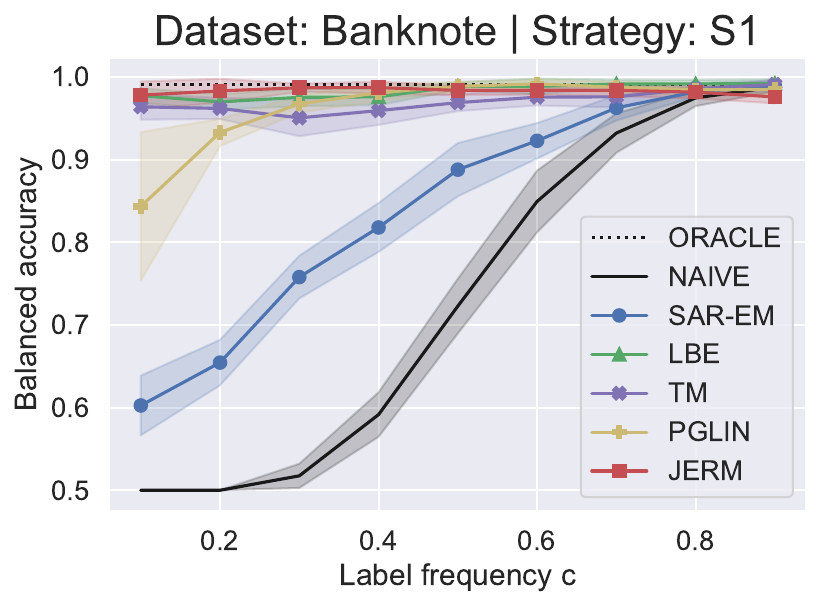} &
\includegraphics[width=0.45\textwidth]{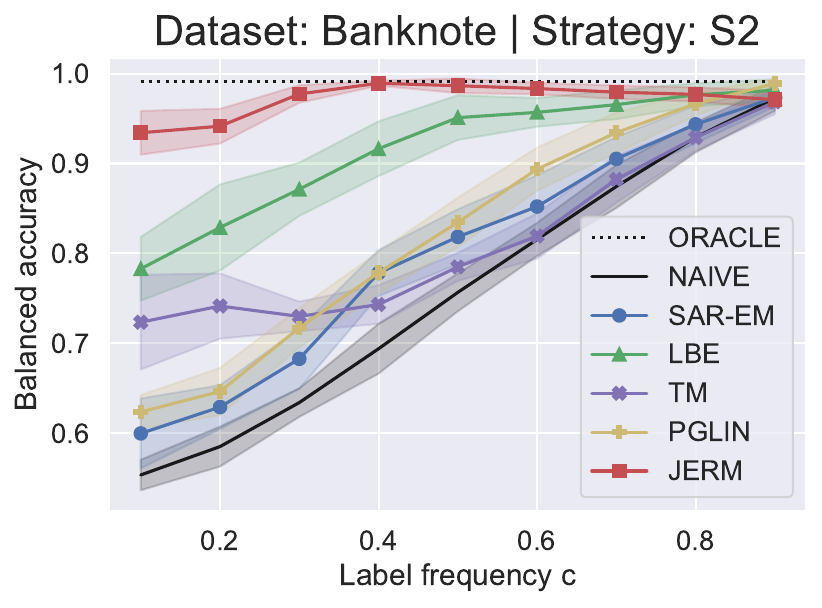} \\
\includegraphics[width=0.45\textwidth]{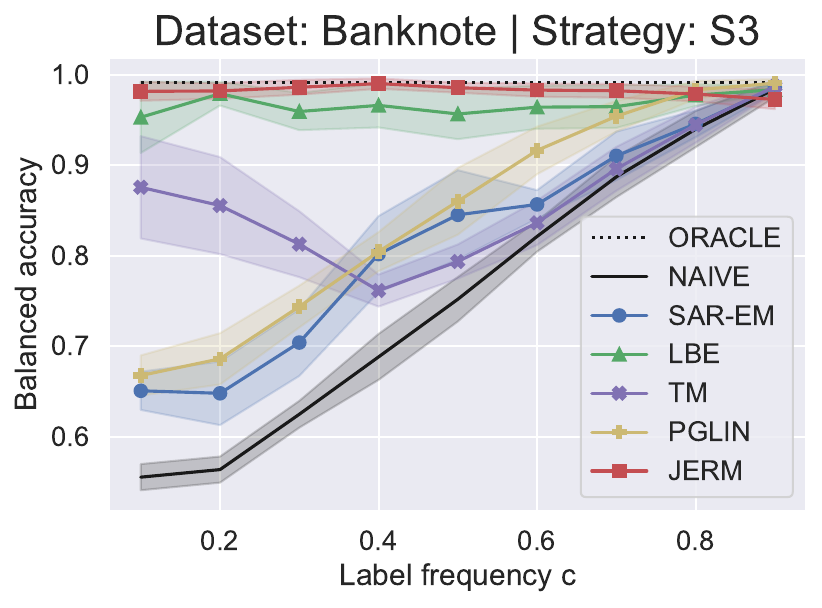}&
\includegraphics[width=0.45\textwidth]{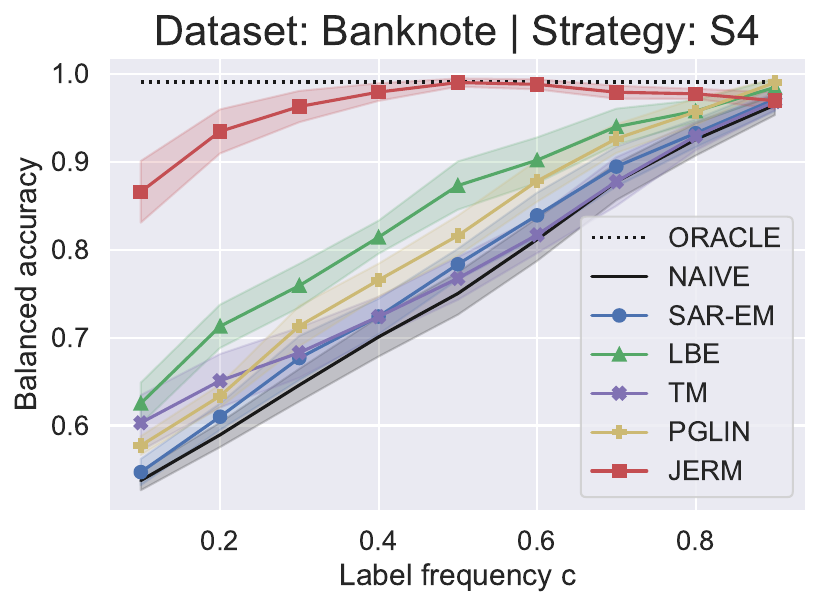} \\
\end{array}$
\caption{Balanced accuracy with respect to label frequency $c$ for selected tabular datasetes and  strategies S1-S4.}
\label{Res_cplots1}
\end{figure}

\begin{figure}[ht!]
\centering
$\begin{array}{cc}
\includegraphics[width=0.45\textwidth]{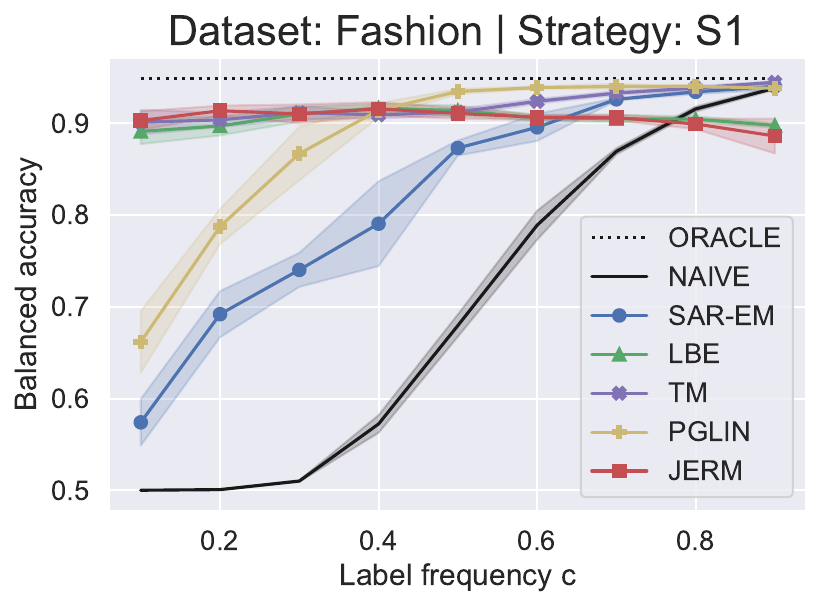} &
\includegraphics[width=0.45\textwidth]{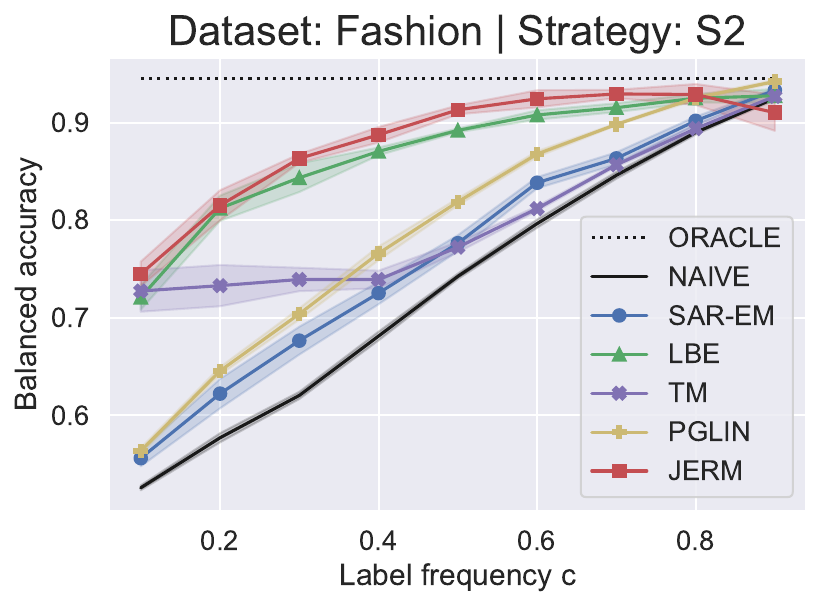} \\
\includegraphics[width=0.45\textwidth]{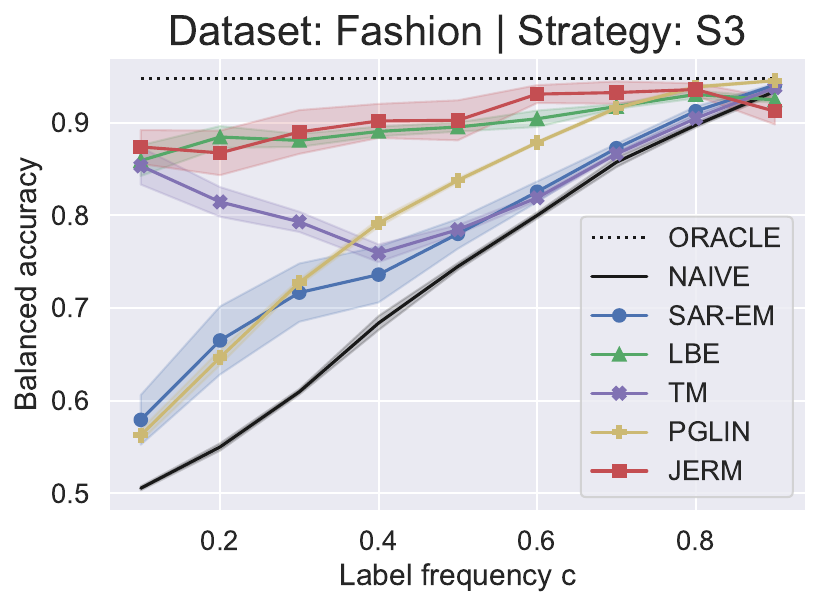}&
\includegraphics[width=0.45\textwidth]{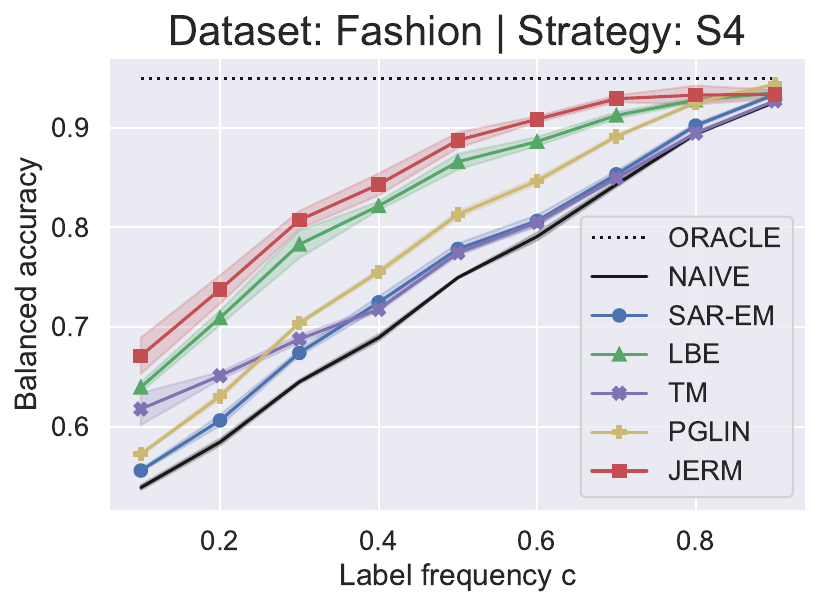} \\
\includegraphics[width=0.45\textwidth]{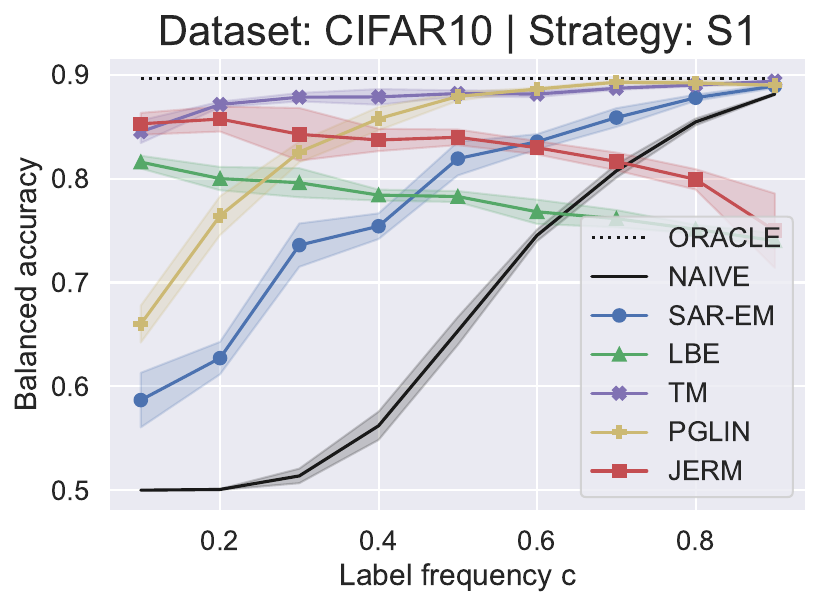} &
\includegraphics[width=0.45\textwidth]{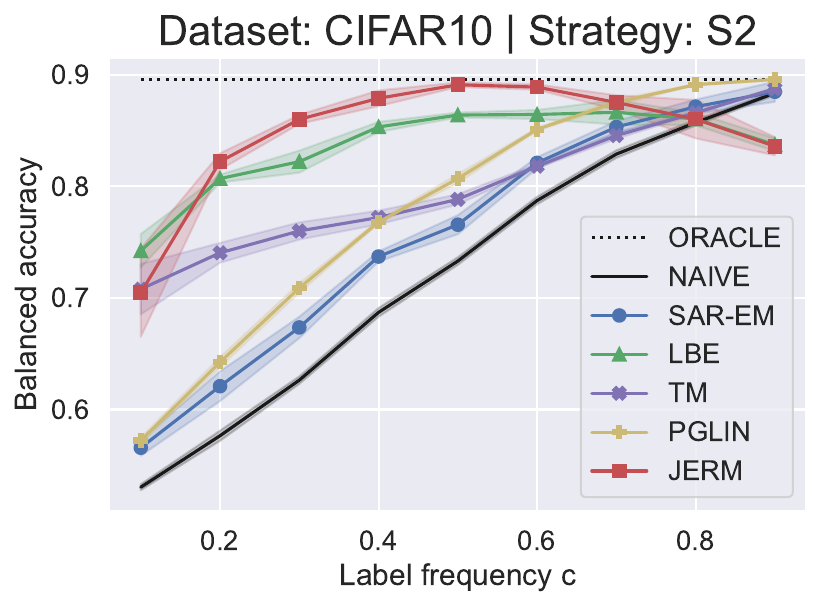} \\
\includegraphics[width=0.45\textwidth]{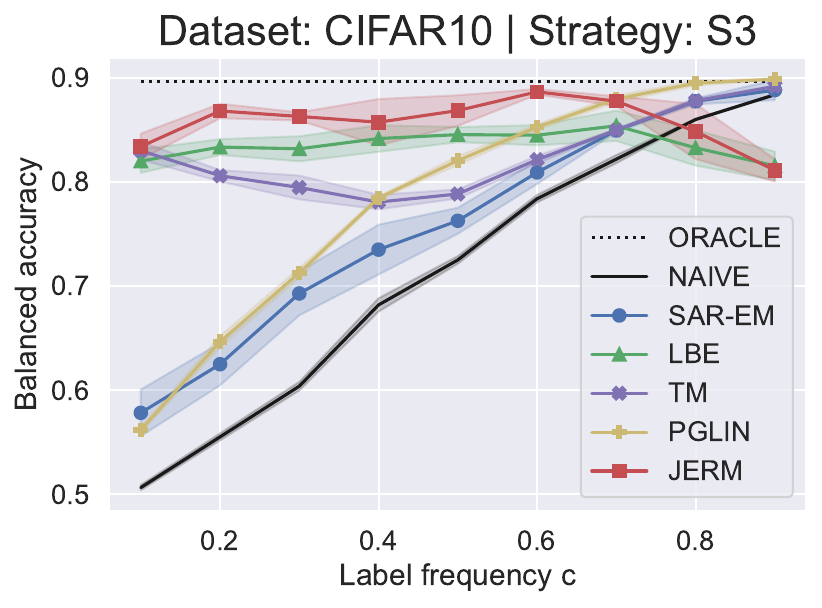}&
\includegraphics[width=0.45\textwidth]{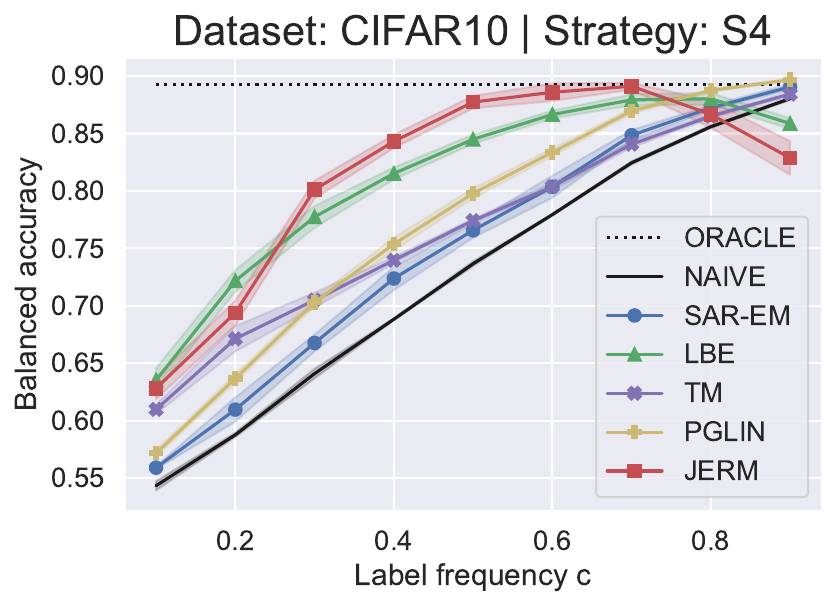} \\
\end{array}$
\caption{Balanced accuracy with respect to label frequency $c$ for selected image datasetes and  strategies S1-S4.}
\label{Res_cplots2}
\end{figure}

\begin{table}[ht!]
\centering
{\small
\caption{Wins (W),  losses (L) and draws (D) of the proposed method \textsc{JERM} 
against each competitive method in terms of average balanced accuracy.
Win/Loss  means that the difference in average accuracy is greater/smaller for the proposed method and intervals mean$\pm$sd  for the pair of compared methods do not overlap. A draw (D) means that the intervals mean$\pm$sd  for the pair of compared methods overlap.}
\label{tab:general}
\begin{tabular}{l|l|ccc|ccc|ccc|ccc|ccc}
\toprule 
&	&  & \textsc{NAIVE}  &  	&&	\textsc{EM} & & &	\textsc{LBE} & &	&	\textsc{PGLIN} & &	&	\textsc{TM}  & \\
Strategy&	c& W &  L & D	&W&	 L&  D&  W&	 L&  D&	 W&	 L&  D&	W&	 L &  D	\\
\midrule
&	0.3&	19&	0&	1&	17&	0&	3&	1&	1&	18&	8&	2&	10&	3&	5&	12\\
S1&	0.5&	18&	0&	2&	11&	0&	9&	2&	2&	16&	3&	6&	11&	1&	5&	14\\
&	0.7&	13&	2&	5&	8&	6&	6&	3&	2&	15&	1&	8&	11&	1&	7&	12\\
\midrule
&	0.3&	20&	0&	0&	18&	0&	2&	3&	1&	16&	13&	0&	7&	14&	0&	6\\
S2&	0.5&	18&	0&	2&	16&	0&	4&	4&	2&	14&	9&	1&	10&	14&	0&	6\\
&	0.7&	15&	0&	5&	11&	1&	8&	3&	1&	16&	5&	3&	12&	11&	1&	8\\
\midrule
&	0.3&	20&	0&	0&	18&	0&	2&	1&	2&	17&	13&	2&	5&	9&	0&	11\\
S3&	0.5&	17&	0&	3&	14&	0&	6&	2&	1&	17&	7&	2&	11&	9&	0&	11\\
&	0.7&	12&	0&	8&	10&	0&	10&	1&	1&	18&	5&	4&	11&	8&	0&	12\\
\midrule
&	0.3&	19&	0&	1&	18&	0&	2&	3&	0&	17&	11&	0&	9&	15&	0&	5\\
S4&	0.5&	17&	0&	3&	15&	0&	5&	5&	0&	15&	10&	1&	9&	13&	0&	7\\
&	0.7&	14&	0&	6&	10&	1&	9&	3&	0&	17&	5&	2&	13&	10&	0&	10\\
\bottomrule
\end{tabular}
}
\end{table}

\begin{table}[ht!]
\centering
{\small
\caption{Percentage of absolute wins (averaged over $c$) of the \textsc{JERM} method
against each competitive method.
}
\label{tab:general_simple}
\begin{tabular}{l|ccccc}
\toprule 
Strategy & \textsc{NAIVE}  &  \textsc{EM}  &	\textsc{LBE} &	\textsc{PGLIN} 	&	\textsc{TM}   \\
\midrule
 S1&  0.91 &0.8	& 0.42	&0.48	&0.42 \\
S2 & 0.95&	0.9&	0.63&	0.71&	0.87 \\
S3& 0.97&	0.87&	0.62&	0.65&	0.8\\
S4& 0.97&	0.9&	0.65&	0.77&	0.87\\
\bottomrule
\end{tabular}
}
\end{table}

\section{Conclusions}
In this work, we proposed a novel PU learning method JERM, based on joint modeling the posterior probability and the propensity score  using sigmoid functions which depend on the linear combination of features.
The parameters of the above functions are estimated by optimizing the joint risk function for the observed label indicator.
The risk function minimizer was analyzed theoretically. In particular, we established  risk consistency for the empirical risk minimizer.

We also proposed an iterative  method for optimizing joint empirical risk. The algorithm consists in alternately determining estimates for the posterior probability and the propensity score function. An innovative element of the work is the use of the spies technique to estimate the set of positive observations, which in turn enables the determination of the propensity score estimator.
Importantly, in future research, the proposed propensity score estimator can be combined with other PU methods that require knowledge of the propensity scores.
The results of experiments conducted on 20 real data sets indicate that the proposed method works comparable to or better than state-of-the-art methods based on propensity score estimation.
The advantage of our method is especially noticeable when the SCAR assumption is not met and the labeling frequency is low.
Moreover, we have shown that the method  is robust with respect to the form of labeling.
Interesting future research direction  is  investigation   of effectiveness of selecting active predictors under the considered scenario by augmenting  joint empirical risk considered here with  sparsity-inducing regularizers.



\bibliographystyle{elsarticle-num} 
\bibliography{References}

\begin{thebibliography}{10}
\expandafter\ifx\csname url\endcsname\relax
  \def\url#1{\texttt{#1}}\fi
\expandafter\ifx\csname urlprefix\endcsname\relax\def\urlprefix{URL }\fi
\expandafter\ifx\csname href\endcsname\relax
  \def\href#1#2{#2} \def\path#1{#1}\fi

\bibitem{BekkerDavis2020}
J.~Bekker, J.~Davis, Learning from positive and unlabeled data: a survey,
  Machine Learning 109 (2020) 719--760.

\bibitem{Chapelle2010}
O.~Chapelle, B.~Sch{\"o}lkopf, A.~Zien, Semi-Supervised Learning, The MIT
  Press, 2010.

\bibitem{Park2016}
J.~W. Park, M.~K. Chu, J.~M. Kim, S.~G. Park, S.~J. Cho, Analysis of trigger
  factors in episodic migraineurs using a smartphone headache diary
  applications, PloS one 11~(2) (2016) 1--13.

\bibitem{LiLiu2003}
X.~Li, B.~Liu, Learning to classify texts using positive and unlabeled data,
  in: Proceedings of the 18th International Joint Conference on Artificial
  Intelligence, IJCAI'03, 2003, p. 587–592.

\bibitem{Fung2006}
G.~P.~C. Fung, J.~X. Yu, H.~Lu, P.~S. Yu, Text classification without negative
  examples revisit, IEEE Transactions on Knowledge and Data Engineering 18~(1)
  (2006) 6–20.

\bibitem{Chiaroni2018}
F.~Chiaroni, M.-C. Rahal, N.~Hueber, F.~Dufaux, Learning with a generative
  adversarial network from a positive unlabeled dataset for image
  classification, in: Proceedings of the 25th IEEE International Conference on
  Image Processing, ICIP'18, 2018, pp. 1--6.

\bibitem{Luo2018}
Y.~Luo, S.~Cheng, C.~Liu, , F.~Jiang, Pu-learning in payload-based web anomaly
  detection, in: Proceedings of the Third Conference on Security of Smart
  Cities, industrial Control Systems and Communications, SSIC'2018, 2018, pp.
  1--5.

\bibitem{XMLC}
E.~Shultheis, R.~Babbar, M.~Wydmuch, K.~Dembczy\'nski, On missing labels,
  long-tails and propensities in extreme multi-label classification, in:
  KDD'22, 2022, pp. 1547--1557.

\bibitem{Li2021}
F.~Li, S.~Dong, A.~Leier, M.~Han, X.~Guo, J.~Xu, X.~Wang, S.~Pan, C.~Jia,
  Y.~Zhang, G.~Webb, L.~J.~M. Coin, C.~Li, J.~Song, Positive-unlabeled learning
  in bioinformatics and computational biology: a brief review, Briefings in
  Bioinformatics 23~(1) (2021).

\bibitem{Elkan2001}
C.~Elkan, The foundations of cost-sensitive learning, in: Proceedings of the
  seventeenth international joint conference on artificial intelligence,
  Vol.~17, Lawrence Erlbaum Associates Ltd, 2001, pp. 973--978.

\bibitem{ElkanNoto2008}
C.~Elkan, K.~Noto, Learning classifiers from only positive and unlabeled data,
  in: Proceedings of the 14th ACM SIGKDD International Conference on Knowledge
  Discovery and Data Mining, KDD '08, 2008, pp. 213--220.

\bibitem{Ramaswamy2016}
H.~Ramaswamy, C.~Scott, A.~Tewari, Mixture proportion estimation via kernel
  embeddings of distributions, in: Proceedings of The 33rd International
  Conference on Machine Learning, Vol.~48, 2016, pp. 2052--2060.

\bibitem{BekkerAAAI18}
J.~Bekker, J.~Davis, Estimating the class prior in positive and unlabeled data
  through decision tree induction, in: Proceedings of the 32th {AAAI}
  Conference on Artificial Intelligence, 2018, pp. 1--8.

\bibitem{ICCS2020}
P.~Teisseyre, J.~Mielniczuk, M.~{\L}azecka, Different strategies of fitting
  logistic regression for positive and unlabelled data, in: Proceedings of
  Intrernational Conference on Computational Science, ICCS'20, 2020, pp. 1--14.

\bibitem{TEISSEYRE2021}
P.~Teisseyre, Classifier chains for positive unlabelled multi-label learning,
  Knowledge-Based Systems 213 (2021) 106709.

\bibitem{BekkerRobberechtsDavis2019}
J.~Bekker, P.~Robberechts, J.~Davis, {B}eyond the {S}elected {C}ompletely {A}t
  {R}andom {A}ssumption for {L}earning from {P}ositive and {U}nlabeled {D}ata,
  in: {P}roceedings of the 2019 {E}uropean {C}onference on {M}achine {L}earning
  and {P}rinciples and {P}ractice of {K}nowledge {D}iscovery in {D}atabases,
  ECML'19, 2019, pp. 71--85.

\bibitem{Gong2021}
C.~Gong, Q.~Wang, T.~Liu, B.~Han, J.~You, J.~Yang, D.~Tao, Instance-dependent
  positive and unlabeled learning with labeling bias estimation, IEEE Trans
  Pattern Anal Mach Intell (2021) 1--16.

\bibitem{Gerych2022}
W.~Gerych, T.~Hartvigsen, L.~Buquicchio, E.~Agu, E.~Rundensteiner, Recovering
  the propensity score from biased positive unlabeled data, in: Proceedings of
  the AAAI Conference on Artificial Intelligence, AAAI'22, 2022, pp.
  6694--6702.

\bibitem{FurmanczykECAI2023}
K.~Furma\'nczyk, J.~Mielniczuk, W.~Rejchel, P.~Teisseyre, Double logistic
  regression approach to biased positive-unlabeled data, in: Proceedings of the
  European Conference on Artificial Intelligence, ECAI'23, 2023, pp. 764--771.

\bibitem{NaVAE}
B.~Na, H.~Kim, K.~Song, W.~Joo, Y.-Y. Kim, I.~Moon, Deep generative
  positive-unlabeled learning under selection bias, in: Proceedings of CIKM'20,
  CIKM '20, ACM, New York, NY, USA, 2020, pp. 1155--–1164.

\bibitem{VAEPUCC}
A.~Wawrzeńczyk, J.~Mielniczuk, One-class classification approach to
  variational learning from biased positive unlabelled data, in: Proceedings of
  the European Conference on Artificial Intelligence, ECAI'23, 2023, pp.
  1720--1727.

\bibitem{Menon2018}
A.~Menon, B.~Rooyen, N.~Natarajan, Learning from binsary labels with
  instant-dependent noise, Machine Learning (2018) 1561--1595.

\bibitem{Cannings2020}
T.~Cannings, Y.~Fan, R.~Samworth, Classification with imperfect training
  labels, Biometrika (2020) 311--330.

\bibitem{Gong2022}
C.~Gong, M.~I. Zulfiqar, C.~Zhang, S.~Mahmood, J.~Yang, A recent survey on
  instance-dependent positive and unlabeled learning, Fundamental Research
  (2022).

\bibitem{He2018}
F.~He, T.~Liu, J.~Webb, D.~Tao, \href{arXiv:180802180}{Instance-dependent pu
  learning by {B}ayesian optimal relabeling}, manuscript (2018).
\newline\urlprefix\url{arXiv:180802180}

\bibitem{Coudray2023}
O.~Coudray, C.~Keribin, P.~Massart, P.~Pamphile, Risk bounds for
  positive-unlabeled learning under the selected at random ssumption, Journal
  of Machine Learning Research (2023) 1--31.

\bibitem{duPlessis2014}
M.~C. du~Plessis, G.~Niu, M.~Sugiyama, Analysis of learning from positive and
  unlabeled data, in: Proceedings of the International Conference on Neural
  Information Processing Systems, NIPS'14, 2014, pp. 703--711.

\bibitem{Kiryo2017}
R.~Kiryo, G.~Niu, M.~C. du~Plessis, M.~Sugiyama, Positive-unlabeled learning
  with non-negative risk estimator, in: Proceedings of the International
  Conference on Neural Information Processing Systems, NIPS'17, 2017, pp.
  1674--1684.

\bibitem{ChenLiuWangZhaoWu2020}
H.~Chen, F.~Liu, Y.~Wang, L.~Zhao, H.~Wu, A variational approach for learning
  from positive and unlabeled data, in: Proceedings of the International
  Conference on Neural Information Processing Systems, NIPS'20, 2020, pp.
  14844--14854.

\bibitem{Zhao2022}
Y.~Zhao, Q.~Xu, Y.~Jiang, P.~Wen, Q.~Huang, Dist-pu: Positive-unlabeled
  learning from a label distribution perspective, in: Proceedings of the
  Conference on Computer Vision and Pattern Recognition, CVPR'22, 2022, pp.
  14461--14470.

\bibitem{LIU2023}
Y.~Liu, J.~Zhao, Y.~Xu, Robust and unbiased positive and unlabeled learning,
  Knowledge-Based Systems 277 (2023) 1--15.

\bibitem{Song2019}
H.~Song, G.~Raskutti, Pu-lasso: high-dimensional varaiable selection with
  presence-only data, Journal of Americal Statistical Association 115 (2019)
  334--347.

\bibitem{LiuLeeYu2002}
B.~Liu, W.~S. Lee, P.~S. Yu, X.~Li, Partially supervised classification of text
  documents, in: Proceedings of the 19-th International Conference on Machine
  Learning, ICLM'02, 2002, pp. 387--394.

\bibitem{Lazeckaetal2021}
M.~{\L}azecka, J.~Mielniczuk, P.~Teisseyre, Estimating the class prior for
  positive and unlabelled data via logistic regression, Advances in Data
  Analysis and Classification 15 (2021) 1039--1068.

\bibitem{Wainwright2019}
M.~Wainwright, High-dimensional Statistics, Cambridge University Press, 2019.

\bibitem{RW2011}
M.~Reid, R.~Williamson, Information divergence and risk for binary experiments,
  Journal of Machine Learning Research 12 (2011) 731--817.

\bibitem{Maurer2016}
A.~Maurer, A vector-contraction inequality for {R}ademacher complexities, in:
  Algorithmic Learning Theory, ALT'16, 2016, pp. 3--17.

\bibitem{Ledoux1991}
M.~Ledoux, M.~Talagrand, Probability in Banach Spaces: Isoperimetry and
  Processes, Springer, Berlin, 1991.

\bibitem{vaart1998}
A.~v.~d. Vaart, {A}symptotic {S}tatistics, Cambridge University Press,
  Cambridge, 1998.

\bibitem{Devroye2012}
L.~Devroye, G.~Lugosi, Combinatorial Methods in Density Estimation, Springer,
  New York, 2012.

\bibitem{Hastie2015}
T.~Hastie, R.~Tibshirani, M.~Wainwright, {S}tatistical {L}earning with
  {S}parsity: {T}he {L}asso and {G}eneralizations, CRC, London, 2015.

\bibitem{PyTorch19}
A.~Paszke, S.~Gross, F.~Massa, A.~Lerer, J.~Bradbury, G.~Chanan, T.~Killeen,
  Z.~Lin, N.~Gimelshein, L.~Antiga, A.~Desmaison, A.~Kopf, E.~Yang, Z.~DeVito,
  M.~Raison, A.~Tejani, S.~Chilamkurthy, B.~Steiner, L.~Fang, J.~Bai,
  S.~Chintala, Pytorch: An imperative style, high-performance deep learning
  library, in: Advances in Neural Information Processing Systems, NIPS'19,
  2019, pp. 8024--8035.

\end{thebibliography}

\clearpage

\appendix

\section{}
\setcounter{table}{0}

\begin{table}[ht!]
\centering
\caption{Statistics of the considered data sets.}
\label{Tab:datasets}
\begin{tabular}{l|lllll}
\toprule 
                Dataset    & n &   p & $P(Y=1)$ & positives & type\\
                \midrule
                  Breast-w &   699  &  9  &0.34 &    241.0 & tabular\\
                  Diabetes &   768  &  8  &0.35 &    268.0& tabular\\
                  Spambase &  4601  & 57  &0.39 &   1813.0& tabular\\
                      Wdbc &   569  & 30  &0.37 &    212.0& tabular\\
   Banknote &  1372  &  4  &0.44 &    610.0& tabular\\
             Heart &   270  & 13  &0.44 &    120.0& tabular\\
                Ionosphere &   351  & 34  &0.64 &    225.0& tabular\\
                     Sonar &   208  & 60  &0.53 &    111.0& tabular\\
                 Haberman  &  306   & 3  &0.26  &    81.0& tabular\\
                  Segment  & 2310   &19  &0.14  &   330.0& tabular\\
            Waveform  & 5000   &40  &0.34  &  1692.0& tabular\\
                    Yeast  & 1484   & 8  &0.31  &   463.0& tabular\\
                     Musk  & 6598  &166  &0.15  &  1017.0& tabular\\
                   Isolet  & 7797  &617  &0.04  &   300.0& tabular\\
                  Semeion  & 1593  &256  & 0.1  &   162.0& tabular\\
                  Vehicle  &  846  & 18  &0.26 &    218.0& tabular\\
\midrule
                  CIFAR10  & 50000   &  - & 0.4 & 20000   & images\\
                  MNIST  &  60000  & -  & 0.49&  29492  & images\\
                  USPS  &  7291  &  - &0.58 & 4240   & images\\
                  Fashion  & 60000   & -  & 0.5 & 30000    & images\\                  
\bottomrule
\end{tabular}
\end{table}

\clearpage

\subsection{Results for labeling strategy S1}

\begin{table}[h!]
\centering
{\tiny
\caption{Balanced accuracy for datasets, for labelling strategy {\bf S1} and $c=0.3$.}
\label{tab:S1c03}
\begin{tabular}{l|l|llllll}
\toprule 
Dataset &	\textsc{ORACLE} &	\textsc{NAIVE} &	\textsc{SAR-EM}  &	\textsc{LBE}  & \textsc{PGLIN}  & \textsc{TM} & \textsc{JERM} \\
\midrule
Breast-w &	0.952 $\pm$ 0.017 &	0.529 $\pm$ 0.022 &	0.74 $\pm$ 0.042 &	{\bf 0.906 $\pm$ 0.036 } &	0.868 $\pm$ 0.022 &	0.904 $\pm$ 0.016 &	0.889 $\pm$ 0.028 \\
Diabetes &	0.716 $\pm$ 0.026 &	0.501 $\pm$ 0.002 &	0.519 $\pm$ 0.017 &	0.636 $\pm$ 0.03 &	{\bf 0.717 $\pm$ 0.02 } &	0.687 $\pm$ 0.064 &	0.684 $\pm$ 0.026 \\
Spambase &	0.894 $\pm$ 0.009 &	0.512 $\pm$ 0.005 &	0.673 $\pm$ 0.017 &	0.863 $\pm$ 0.016 &	0.672 $\pm$ 0.035 &	{\bf 0.883 $\pm$ 0.008 } &	0.813 $\pm$ 0.031 \\
Wdbc &	0.971 $\pm$ 0.01 &	0.565 $\pm$ 0.036 &	0.693 $\pm$ 0.049 &	0.859 $\pm$ 0.035 &	0.735 $\pm$ 0.042 &	{\bf 0.904 $\pm$ 0.036 } &	0.753 $\pm$ 0.078 \\
Banknote &	0.991 $\pm$ 0.005 &	0.519 $\pm$ 0.009 &	0.759 $\pm$ 0.024 &	0.974 $\pm$ 0.01 &	0.959 $\pm$ 0.027 &	0.964 $\pm$ 0.016 &	{\bf 0.983 $\pm$ 0.007 } \\
Heart &	0.808 $\pm$ 0.038 &	0.512 $\pm$ 0.014 &	0.6 $\pm$ 0.07 &	0.683 $\pm$ 0.073 &	{\bf 0.774 $\pm$ 0.029 } &	0.771 $\pm$ 0.046 &	0.657 $\pm$ 0.052 \\
Ionosphere &	0.82 $\pm$ 0.022 &	0.5 $\pm$ 0.007 &	0.508 $\pm$ 0.048 &	0.664 $\pm$ 0.081 &	{\bf 0.792 $\pm$ 0.047 } &	0.784 $\pm$ 0.052 &	0.726 $\pm$ 0.05 \\
Sonar &	0.778 $\pm$ 0.051 &	0.502 $\pm$ 0.005 &	0.533 $\pm$ 0.045 &	{\bf 0.684 $\pm$ 0.082 } &	0.607 $\pm$ 0.069 &	0.616 $\pm$ 0.055 &	0.574 $\pm$ 0.093 \\
Haberman &	0.511 $\pm$ 0.019 &	0.5 $\pm$ 0.0 &	0.499 $\pm$ 0.003 &	0.532 $\pm$ 0.045 &	0.511 $\pm$ 0.019 &	0.502 $\pm$ 0.005 &	{\bf 0.573 $\pm$ 0.064 } \\
Segment &	0.985 $\pm$ 0.006 &	0.515 $\pm$ 0.013 &	0.614 $\pm$ 0.035 &	{\bf 0.98 $\pm$ 0.008 } &	0.861 $\pm$ 0.055 &	0.932 $\pm$ 0.023 &	0.979 $\pm$ 0.008 \\
Waveform &	0.841 $\pm$ 0.015 &	0.511 $\pm$ 0.004 &	0.691 $\pm$ 0.026 &	0.833 $\pm$ 0.009 &	0.729 $\pm$ 0.023 &	{\bf 0.84 $\pm$ 0.011 } &	0.82 $\pm$ 0.012 \\
Yeast &	0.518 $\pm$ 0.01 &	0.5 $\pm$ 0.0 &	0.501 $\pm$ 0.002 &	0.623 $\pm$ 0.021 &	0.536 $\pm$ 0.024 &	0.576 $\pm$ 0.032 &	{\bf 0.658 $\pm$ 0.019 } \\
Musk &	0.761 $\pm$ 0.012 &	0.512 $\pm$ 0.011 &	0.565 $\pm$ 0.018 &	0.83 $\pm$ 0.014 &	0.685 $\pm$ 0.022 &	0.805 $\pm$ 0.016 &	{\bf 0.833 $\pm$ 0.014 } \\
Isolet &	0.62 $\pm$ 0.03 &	0.5 $\pm$ 0.0 &	0.569 $\pm$ 0.032 &	{\bf 0.897 $\pm$ 0.009 } &	0.645 $\pm$ 0.035 &	0.85 $\pm$ 0.029 &	0.896 $\pm$ 0.027 \\
Semeion &	0.819 $\pm$ 0.029 &	0.504 $\pm$ 0.009 &	0.569 $\pm$ 0.052 &	{\bf 0.851 $\pm$ 0.039 } &	0.808 $\pm$ 0.049 &	0.851 $\pm$ 0.02 &	0.836 $\pm$ 0.022 \\
Vehicle &	0.937 $\pm$ 0.017 &	0.527 $\pm$ 0.021 &	0.671 $\pm$ 0.054 &	{\bf 0.904 $\pm$ 0.022 } &	0.799 $\pm$ 0.059 &	0.844 $\pm$ 0.033 &	0.878 $\pm$ 0.08 \\
CIFAR10 &	0.893 $\pm$ 0.0 &	0.512 $\pm$ 0.003 &	0.726 $\pm$ 0.031 &	0.808 $\pm$ 0.011 &	0.822 $\pm$ 0.018 &	{\bf 0.88 $\pm$ 0.002 } &	0.849 $\pm$ 0.007 \\
MNIST &	0.858 $\pm$ 0.0 &	0.51 $\pm$ 0.003 &	0.606 $\pm$ 0.019 &	0.787 $\pm$ 0.016 &	{\bf 0.795 $\pm$ 0.018 } &	0.749 $\pm$ 0.024 &	0.767 $\pm$ 0.007 \\
USPS &	0.877 $\pm$ 0.0 &	0.505 $\pm$ 0.0 &	0.666 $\pm$ 0.021 &	0.775 $\pm$ 0.005 &	0.812 $\pm$ 0.03 &	{\bf 0.83 $\pm$ 0.007 } &	0.778 $\pm$ 0.01 \\
Fashion &	0.948 $\pm$ 0.0 &	0.509 $\pm$ 0.003 &	0.766 $\pm$ 0.009 &	0.906 $\pm$ 0.009 &	0.872 $\pm$ 0.021 &	{\bf 0.913 $\pm$ 0.004 } &	0.907 $\pm$ 0.008 \\
\midrule
Avg. rank &	0.0 &	1.05 &	2.0 &	4.8 &	3.8 &	4.8 &	4.55 \\
\bottomrule
\end{tabular}
}
\end{table}

\begin{table}[h!]
\centering
{\tiny
\caption{Balanced accuracy for datasets, for labelling strategy {\bf S1} and $c=0.5$.}
\label{tab:S1c05}
% SCAR c=0.1
\begin{tabular}{l|l|llllll}
\toprule 
Dataset &	\textsc{ORACLE} &	\textsc{NAIVE} &	\textsc{SAR-EM}  &	\textsc{LBE}  & \textsc{PGLIN}  & \textsc{TM} & \textsc{JERM} \\
\midrule
Breast-w &	0.952 $\pm$ 0.017 &	0.662 $\pm$ 0.043 &	0.851 $\pm$ 0.023 &	{\bf 0.949 $\pm$ 0.024 } &	0.909 $\pm$ 0.021 &	0.915 $\pm$ 0.024 &	0.948 $\pm$ 0.035 \\
Diabetes &	0.716 $\pm$ 0.026 &	0.51 $\pm$ 0.016 &	0.575 $\pm$ 0.027 &	0.639 $\pm$ 0.03 &	{\bf 0.732 $\pm$ 0.026 } &	0.715 $\pm$ 0.036 &	0.663 $\pm$ 0.027 \\
Spambase &	0.894 $\pm$ 0.009 &	0.605 $\pm$ 0.016 &	0.756 $\pm$ 0.017 &	0.858 $\pm$ 0.014 &	0.843 $\pm$ 0.02 &	{\bf 0.889 $\pm$ 0.01 } &	0.778 $\pm$ 0.013 \\
Wdbc &	0.971 $\pm$ 0.01 &	0.686 $\pm$ 0.054 &	0.789 $\pm$ 0.066 &	0.903 $\pm$ 0.04 &	0.873 $\pm$ 0.03 &	{\bf 0.918 $\pm$ 0.029 } &	0.863 $\pm$ 0.037 \\
Banknote &	0.991 $\pm$ 0.005 &	0.736 $\pm$ 0.038 &	0.884 $\pm$ 0.019 &	{\bf 0.989 $\pm$ 0.007 } &	0.986 $\pm$ 0.01 &	0.97 $\pm$ 0.016 &	0.984 $\pm$ 0.006 \\
Heart &	0.808 $\pm$ 0.038 &	0.618 $\pm$ 0.047 &	0.717 $\pm$ 0.042 &	0.682 $\pm$ 0.066 &	0.77 $\pm$ 0.04 &	{\bf 0.783 $\pm$ 0.042 } &	0.684 $\pm$ 0.083 \\
Ionosphere &	0.82 $\pm$ 0.022 &	0.558 $\pm$ 0.046 &	0.721 $\pm$ 0.067 &	0.69 $\pm$ 0.046 &	0.794 $\pm$ 0.03 &	{\bf 0.818 $\pm$ 0.051 } &	0.733 $\pm$ 0.038 \\
Sonar &	0.778 $\pm$ 0.051 &	0.534 $\pm$ 0.041 &	0.643 $\pm$ 0.069 &	{\bf 0.713 $\pm$ 0.056 } &	0.642 $\pm$ 0.063 &	0.687 $\pm$ 0.055 &	0.616 $\pm$ 0.072 \\
Haberman &	0.511 $\pm$ 0.019 &	0.5 $\pm$ 0.0 &	0.51 $\pm$ 0.018 &	0.541 $\pm$ 0.029 &	0.513 $\pm$ 0.035 &	0.508 $\pm$ 0.028 &	{\bf 0.563 $\pm$ 0.052 } \\
Segment &	0.985 $\pm$ 0.006 &	0.648 $\pm$ 0.038 &	0.816 $\pm$ 0.058 &	{\bf 0.989 $\pm$ 0.006 } &	0.951 $\pm$ 0.017 &	0.954 $\pm$ 0.014 &	0.982 $\pm$ 0.009 \\
Waveform &	0.841 $\pm$ 0.015 &	0.584 $\pm$ 0.017 &	0.74 $\pm$ 0.03 &	0.835 $\pm$ 0.006 &	0.825 $\pm$ 0.013 &	{\bf 0.846 $\pm$ 0.007 } &	0.817 $\pm$ 0.011 \\
Yeast &	0.518 $\pm$ 0.01 &	0.5 $\pm$ 0.0 &	0.511 $\pm$ 0.006 &	0.619 $\pm$ 0.014 &	0.588 $\pm$ 0.044 &	0.622 $\pm$ 0.015 &	{\bf 0.651 $\pm$ 0.021 } \\
Musk &	0.761 $\pm$ 0.012 &	0.55 $\pm$ 0.01 &	0.623 $\pm$ 0.023 &	0.799 $\pm$ 0.009 &	0.734 $\pm$ 0.021 &	0.811 $\pm$ 0.019 &	{\bf 0.821 $\pm$ 0.013 } \\
Isolet &	0.62 $\pm$ 0.03 &	0.516 $\pm$ 0.01 &	0.609 $\pm$ 0.032 &	0.899 $\pm$ 0.012 &	0.685 $\pm$ 0.024 &	0.862 $\pm$ 0.031 &	{\bf 0.909 $\pm$ 0.017 } \\
Semeion &	0.819 $\pm$ 0.029 &	0.555 $\pm$ 0.042 &	0.637 $\pm$ 0.026 &	{\bf 0.89 $\pm$ 0.027 } &	0.835 $\pm$ 0.018 &	0.842 $\pm$ 0.024 &	0.86 $\pm$ 0.016 \\
Vehicle &	0.937 $\pm$ 0.017 &	0.614 $\pm$ 0.033 &	0.776 $\pm$ 0.052 &	{\bf 0.928 $\pm$ 0.014 } &	0.901 $\pm$ 0.027 &	0.877 $\pm$ 0.023 &	0.907 $\pm$ 0.05 \\
CIFAR10 &	0.893 $\pm$ 0.0 &	0.655 $\pm$ 0.012 &	0.809 $\pm$ 0.008 &	0.775 $\pm$ 0.012 &	0.875 $\pm$ 0.003 &	{\bf 0.879 $\pm$ 0.004 } &	0.824 $\pm$ 0.007 \\
MNIST &	0.858 $\pm$ 0.0 &	0.625 $\pm$ 0.01 &	0.738 $\pm$ 0.021 &	0.765 $\pm$ 0.01 &	{\bf 0.847 $\pm$ 0.002 } &	0.832 $\pm$ 0.003 &	0.755 $\pm$ 0.005 \\
USPS &	0.877 $\pm$ 0.0 &	0.645 $\pm$ 0.008 &	0.791 $\pm$ 0.014 &	0.772 $\pm$ 0.009 &	{\bf 0.864 $\pm$ 0.002 } &	0.849 $\pm$ 0.005 &	0.78 $\pm$ 0.011 \\
Fashion &	0.948 $\pm$ 0.0 &	0.675 $\pm$ 0.008 &	0.861 $\pm$ 0.017 &	0.921 $\pm$ 0.005 &	{\bf 0.932 $\pm$ 0.002 } &	0.914 $\pm$ 0.003 &	0.905 $\pm$ 0.004 \\
\midrule
Avg. rank &	0.0 &	1.0 &	2.45 &	4.45 &	4.25 &	4.7 &	4.15 \\
\bottomrule
\end{tabular}
}
\end{table}

\begin{table}[h!]
\centering
{\tiny
\caption{Balanced accuracy for datasets, for labelling strategy {\bf S1} and $c=0.7$.}
\label{tab:S1c07}
% SCAR c=0.1
\begin{tabular}{l|l|lllllll}
\toprule 
Dataset &	\textsc{ORACLE} &	\textsc{NAIVE} &	\textsc{SAR-EM}  &	\textsc{LBE}  & \textsc{PGLIN}  & \textsc{TM} & \textsc{JERM} \\
\midrule
Breast-w &	0.952 $\pm$ 0.017 &	0.846 $\pm$ 0.014 &	0.9 $\pm$ 0.017 &	{\bf 0.959 $\pm$ 0.022 } &	0.94 $\pm$ 0.017 &	0.919 $\pm$ 0.023 &	0.95 $\pm$ 0.016 \\
Diabetes &	0.716 $\pm$ 0.026 &	0.581 $\pm$ 0.021 &	0.675 $\pm$ 0.029 &	0.634 $\pm$ 0.025 &	{\bf 0.75 $\pm$ 0.011 } &	0.719 $\pm$ 0.036 &	0.665 $\pm$ 0.03 \\
Spambase &	0.894 $\pm$ 0.009 &	0.757 $\pm$ 0.021 &	0.83 $\pm$ 0.012 &	0.846 $\pm$ 0.01 &	{\bf 0.895 $\pm$ 0.006 } &	0.895 $\pm$ 0.008 &	0.795 $\pm$ 0.022 \\
Wdbc &	0.971 $\pm$ 0.01 &	0.818 $\pm$ 0.048 &	0.855 $\pm$ 0.039 &	0.919 $\pm$ 0.019 &	0.923 $\pm$ 0.028 &	{\bf 0.938 $\pm$ 0.02 } &	0.872 $\pm$ 0.059 \\
Banknote &	0.991 $\pm$ 0.005 &	0.935 $\pm$ 0.024 &	0.958 $\pm$ 0.016 &	{\bf 0.991 $\pm$ 0.006 } &	0.988 $\pm$ 0.005 &	0.98 $\pm$ 0.014 &	0.985 $\pm$ 0.006 \\
Heart &	0.808 $\pm$ 0.038 &	0.753 $\pm$ 0.038 &	0.791 $\pm$ 0.031 &	0.704 $\pm$ 0.064 &	{\bf 0.804 $\pm$ 0.04 } &	0.792 $\pm$ 0.04 &	0.714 $\pm$ 0.034 \\
Ionosphere &	0.82 $\pm$ 0.022 &	0.801 $\pm$ 0.045 &	{\bf 0.831 $\pm$ 0.034 } &	0.719 $\pm$ 0.063 &	0.794 $\pm$ 0.041 &	0.826 $\pm$ 0.036 &	0.766 $\pm$ 0.037 \\
Sonar &	0.778 $\pm$ 0.051 &	0.653 $\pm$ 0.055 &	0.693 $\pm$ 0.065 &	0.722 $\pm$ 0.084 &	0.649 $\pm$ 0.044 &	{\bf 0.74 $\pm$ 0.034 } &	0.641 $\pm$ 0.067 \\
Haberman &	0.511 $\pm$ 0.019 &	0.499 $\pm$ 0.003 &	0.515 $\pm$ 0.025 &	0.51 $\pm$ 0.012 &	0.509 $\pm$ 0.015 &	0.514 $\pm$ 0.034 &	{\bf 0.552 $\pm$ 0.028 } \\
Segment &	0.985 $\pm$ 0.006 &	0.906 $\pm$ 0.04 &	0.934 $\pm$ 0.016 &	{\bf 0.992 $\pm$ 0.004 } &	0.98 $\pm$ 0.009 &	0.972 $\pm$ 0.009 &	0.984 $\pm$ 0.006 \\
Waveform &	0.841 $\pm$ 0.015 &	0.713 $\pm$ 0.014 &	0.782 $\pm$ 0.012 &	0.836 $\pm$ 0.006 &	{\bf 0.865 $\pm$ 0.01 } &	0.851 $\pm$ 0.012 &	0.808 $\pm$ 0.014 \\
Yeast &	0.518 $\pm$ 0.01 &	0.501 $\pm$ 0.002 &	0.526 $\pm$ 0.018 &	0.604 $\pm$ 0.022 &	0.599 $\pm$ 0.041 &	0.627 $\pm$ 0.015 &	{\bf 0.642 $\pm$ 0.015 } \\
Musk &	0.761 $\pm$ 0.012 &	0.638 $\pm$ 0.016 &	0.695 $\pm$ 0.018 &	0.789 $\pm$ 0.007 &	0.786 $\pm$ 0.016 &	{\bf 0.816 $\pm$ 0.013 } &	0.81 $\pm$ 0.021 \\
Isolet &	0.62 $\pm$ 0.03 &	0.547 $\pm$ 0.015 &	0.647 $\pm$ 0.029 &	0.899 $\pm$ 0.01 &	0.728 $\pm$ 0.033 &	0.864 $\pm$ 0.025 &	{\bf 0.908 $\pm$ 0.01 } \\
Semeion &	0.819 $\pm$ 0.029 &	0.656 $\pm$ 0.077 &	0.763 $\pm$ 0.043 &	{\bf 0.885 $\pm$ 0.023 } &	0.844 $\pm$ 0.019 &	0.855 $\pm$ 0.027 &	0.869 $\pm$ 0.018 \\
Vehicle &	0.937 $\pm$ 0.017 &	0.792 $\pm$ 0.069 &	0.885 $\pm$ 0.034 &	0.93 $\pm$ 0.012 &	{\bf 0.943 $\pm$ 0.015 } &	0.899 $\pm$ 0.019 &	0.915 $\pm$ 0.029 \\
CIFAR10 &	0.893 $\pm$ 0.0 &	0.799 $\pm$ 0.003 &	0.856 $\pm$ 0.006 &	0.768 $\pm$ 0.009 &	{\bf 0.89 $\pm$ 0.001 } &	0.882 $\pm$ 0.004 &	0.817 $\pm$ 0.004 \\
MNIST &	0.858 $\pm$ 0.0 &	0.77 $\pm$ 0.008 &	0.819 $\pm$ 0.004 &	0.754 $\pm$ 0.005 &	0.839 $\pm$ 0.002 &	{\bf 0.848 $\pm$ 0.004 } &	0.739 $\pm$ 0.013 \\
USPS &	0.877 $\pm$ 0.0 &	0.821 $\pm$ 0.006 &	0.846 $\pm$ 0.007 &	0.764 $\pm$ 0.006 &	{\bf 0.86 $\pm$ 0.005 } &	0.859 $\pm$ 0.003 &	0.754 $\pm$ 0.014 \\
Fashion &	0.948 $\pm$ 0.0 &	0.869 $\pm$ 0.005 &	0.923 $\pm$ 0.002 &	0.909 $\pm$ 0.004 &	{\bf 0.941 $\pm$ 0.001 } &	0.932 $\pm$ 0.003 &	0.9 $\pm$ 0.006 \\
\midrule
Avg. rank &	0.0 &	1.6 &	3.1 &	3.7 &	4.5 &	4.65 &	3.45 \\
\bottomrule
\end{tabular}
}
\end{table}

%%%%%%%%%%%%%%%%%%%%%%%%%%%%%%%%%%%%%%%%%%%%%%%%%%%%%%%%%%%%%%%%%%%%%%%%%%%%%%%%%%%%%%%%%%%%%%%%%%%%%%%%%%%%%%%%%%%%%%%%%%%%
%%%%%%%%%%%%%%%%%%%%%%%%%%%%%%%%%%%%%%%%%%%%%%%%%%%%%%%%%%%%%%%%%%%%%%%%%%%%%%%%%%%%%%%%%%%%%%%%%%%%%%%%%%%%%%%%%%%%%%%%%%%%
%%%%%%%%%%%%%%%%%%%%%%%%%%%%%%%%%%%%%%%%%%%%%%%%%%%%%%%%%%%%%%%%%%%%%%%%%%%%%%%%%%%%%%%%%%%%%%%%%%%%%%%%%%%%%%%%%%%%%%%%%%%%
\clearpage

\subsection{Results for labeling strategy S2}

\begin{table}[h!]
\centering
{\tiny
\caption{Balanced accuracy for  datasets, for labelling strategy {\bf S2} and $c=0.3$.}
\label{tab:S2c03}
% SCAR c=0.1
\begin{tabular}{l|l|llllll}
\toprule 
Dataset &	\textsc{ORACLE} &	\textsc{NAIVE} &	\textsc{SAR-EM}  &	\textsc{LBE}  & \textsc{PGLIN}  & \textsc{TM} & \textsc{JERM} \\
\midrule
Breast-w &	0.952 $\pm$ 0.017 &	0.597 $\pm$ 0.034 &	0.639 $\pm$ 0.03 &	{\bf 0.837 $\pm$ 0.051 } &	0.717 $\pm$ 0.025 &	0.728 $\pm$ 0.042 &	0.822 $\pm$ 0.045 \\
Diabetes &	0.716 $\pm$ 0.026 &	0.505 $\pm$ 0.006 &	0.552 $\pm$ 0.017 &	0.706 $\pm$ 0.044 &	0.692 $\pm$ 0.031 &	0.684 $\pm$ 0.046 &	{\bf 0.723 $\pm$ 0.031 } \\
Spambase &	0.893 $\pm$ 0.01 &	0.636 $\pm$ 0.015 &	0.648 $\pm$ 0.021 &	{\bf 0.785 $\pm$ 0.022 } &	0.677 $\pm$ 0.014 &	0.667 $\pm$ 0.015 &	0.784 $\pm$ 0.034 \\
Wdbc &	0.971 $\pm$ 0.01 &	0.653 $\pm$ 0.033 &	0.674 $\pm$ 0.031 &	0.705 $\pm$ 0.044 &	0.704 $\pm$ 0.043 &	0.68 $\pm$ 0.034 &	{\bf 0.764 $\pm$ 0.031 } \\
Banknote &	0.991 $\pm$ 0.005 &	0.633 $\pm$ 0.023 &	0.688 $\pm$ 0.023 &	0.895 $\pm$ 0.055 &	0.726 $\pm$ 0.02 &	0.718 $\pm$ 0.032 &	{\bf 0.977 $\pm$ 0.007 } \\
Heart &	0.808 $\pm$ 0.038 &	0.562 $\pm$ 0.039 &	0.623 $\pm$ 0.025 &	0.703 $\pm$ 0.075 &	{\bf 0.726 $\pm$ 0.047 } &	0.679 $\pm$ 0.058 &	0.682 $\pm$ 0.052 \\
Ionosphere &	0.82 $\pm$ 0.022 &	0.507 $\pm$ 0.013 &	0.545 $\pm$ 0.036 &	0.739 $\pm$ 0.034 &	{\bf 0.767 $\pm$ 0.05 } &	0.711 $\pm$ 0.052 &	0.717 $\pm$ 0.059 \\
Sonar &	0.782 $\pm$ 0.049 &	0.502 $\pm$ 0.007 &	0.568 $\pm$ 0.042 &	0.673 $\pm$ 0.074 &	{\bf 0.683 $\pm$ 0.073 } &	0.583 $\pm$ 0.044 &	0.64 $\pm$ 0.041 \\
Haberman &	0.511 $\pm$ 0.019 &	0.5 $\pm$ 0.0 &	0.507 $\pm$ 0.017 &	{\bf 0.58 $\pm$ 0.06 } &	0.497 $\pm$ 0.039 &	0.499 $\pm$ 0.019 &	0.575 $\pm$ 0.038 \\
Segment &	0.985 $\pm$ 0.006 &	0.556 $\pm$ 0.015 &	0.65 $\pm$ 0.037 &	0.93 $\pm$ 0.04 &	0.774 $\pm$ 0.037 &	0.855 $\pm$ 0.035 &	{\bf 0.933 $\pm$ 0.022 } \\
Waveform &	0.839 $\pm$ 0.014 &	0.566 $\pm$ 0.009 &	0.631 $\pm$ 0.031 &	0.826 $\pm$ 0.017 &	0.666 $\pm$ 0.016 &	0.771 $\pm$ 0.016 &	{\bf 0.85 $\pm$ 0.014 } \\
Yeast &	0.518 $\pm$ 0.01 &	0.5 $\pm$ 0.0 &	0.504 $\pm$ 0.004 &	0.65 $\pm$ 0.016 &	0.632 $\pm$ 0.043 &	0.61 $\pm$ 0.019 &	{\bf 0.69 $\pm$ 0.017 } \\
Musk &	0.764 $\pm$ 0.014 &	0.564 $\pm$ 0.015 &	0.61 $\pm$ 0.024 &	{\bf 0.827 $\pm$ 0.022 } &	0.651 $\pm$ 0.022 &	0.743 $\pm$ 0.024 &	0.797 $\pm$ 0.014 \\
Isolet &	0.603 $\pm$ 0.023 &	0.511 $\pm$ 0.006 &	0.542 $\pm$ 0.01 &	{\bf 0.862 $\pm$ 0.023 } &	0.588 $\pm$ 0.027 &	0.775 $\pm$ 0.039 &	0.814 $\pm$ 0.039 \\
Semeion &	0.787 $\pm$ 0.041 &	0.548 $\pm$ 0.068 &	0.673 $\pm$ 0.03 &	0.765 $\pm$ 0.023 &	0.748 $\pm$ 0.033 &	0.747 $\pm$ 0.036 &	{\bf 0.796 $\pm$ 0.036 } \\
Vehicle &	0.937 $\pm$ 0.017 &	0.576 $\pm$ 0.032 &	0.639 $\pm$ 0.058 &	0.835 $\pm$ 0.036 &	0.707 $\pm$ 0.042 &	0.733 $\pm$ 0.046 &	{\bf 0.85 $\pm$ 0.052 } \\
CIFAR10 &	0.893 $\pm$ 0.0 &	0.633 $\pm$ 0.002 &	0.686 $\pm$ 0.013 &	0.833 $\pm$ 0.011 &	0.708 $\pm$ 0.004 &	0.757 $\pm$ 0.01 &	{\bf 0.855 $\pm$ 0.014 } \\
MNIST &	0.853 $\pm$ 0.0 &	0.609 $\pm$ 0.007 &	0.671 $\pm$ 0.016 &	{\bf 0.813 $\pm$ 0.007 } &	0.7 $\pm$ 0.004 &	0.732 $\pm$ 0.008 &	0.794 $\pm$ 0.015 \\
USPS &	0.879 $\pm$ 0.0 &	0.622 $\pm$ 0.001 &	0.679 $\pm$ 0.009 &	{\bf 0.843 $\pm$ 0.006 } &	0.706 $\pm$ 0.007 &	0.709 $\pm$ 0.004 &	0.804 $\pm$ 0.015 \\
Fashion &	0.948 $\pm$ 0.0 &	0.621 $\pm$ 0.003 &	0.683 $\pm$ 0.012 &	0.846 $\pm$ 0.011 &	0.711 $\pm$ 0.004 &	0.722 $\pm$ 0.009 &	{\bf 0.866 $\pm$ 0.008 } \\
\midrule
Avg. rank &	0.0 &	1.1 &	2.1 &	5.35 &	3.65 &	3.45 &	5.35 \\
\bottomrule
\end{tabular}
}
\end{table}

\begin{table}[h!]
\centering
{\tiny
\caption{Balanced accuracy for  datasets, for labelling strategy {\bf S2} and $c=0.5$.}
\label{tab:S2c05}
\begin{tabular}{l|l|llllll}
\toprule 
Dataset &	\textsc{ORACLE} &	\textsc{NAIVE} &	\textsc{SAR-EM}  &	\textsc{LBE}  & \textsc{PGLIN}  & \textsc{TM} & \textsc{JERM} \\
\midrule
Breast-w &	0.952 $\pm$ 0.017 &	0.722 $\pm$ 0.035 &	0.772 $\pm$ 0.038 &	0.886 $\pm$ 0.034 &	0.849 $\pm$ 0.035 &	0.781 $\pm$ 0.031 &	{\bf 0.902 $\pm$ 0.028 } \\
Diabetes &	0.716 $\pm$ 0.026 &	0.56 $\pm$ 0.02 &	0.626 $\pm$ 0.028 &	0.662 $\pm$ 0.038 &	{\bf 0.722 $\pm$ 0.031 } &	0.685 $\pm$ 0.035 &	0.707 $\pm$ 0.045 \\
Spambase &	0.893 $\pm$ 0.01 &	0.726 $\pm$ 0.016 &	0.748 $\pm$ 0.013 &	{\bf 0.871 $\pm$ 0.009 } &	0.771 $\pm$ 0.015 &	0.748 $\pm$ 0.017 &	0.84 $\pm$ 0.07 \\
Wdbc &	0.971 $\pm$ 0.01 &	0.751 $\pm$ 0.054 &	0.783 $\pm$ 0.033 &	0.789 $\pm$ 0.036 &	0.805 $\pm$ 0.038 &	0.76 $\pm$ 0.037 &	{\bf 0.85 $\pm$ 0.04 } \\
Banknote&	0.991 $\pm$ 0.005 &	0.753 $\pm$ 0.024 &	0.789 $\pm$ 0.026 &	0.929 $\pm$ 0.039 &	0.831 $\pm$ 0.021 &	0.782 $\pm$ 0.022 &	{\bf 0.99 $\pm$ 0.007 } \\
Heart &	0.808 $\pm$ 0.038 &	0.683 $\pm$ 0.04 &	0.696 $\pm$ 0.045 &	0.731 $\pm$ 0.037 &	{\bf 0.777 $\pm$ 0.025 } &	0.728 $\pm$ 0.043 &	0.707 $\pm$ 0.038 \\
Ionosphere &	0.82 $\pm$ 0.022 &	0.62 $\pm$ 0.063 &	0.704 $\pm$ 0.053 &	0.748 $\pm$ 0.062 &	{\bf 0.82 $\pm$ 0.043 } &	0.773 $\pm$ 0.049 &	0.757 $\pm$ 0.043 \\
Sonar &	0.782 $\pm$ 0.049 &	0.56 $\pm$ 0.033 &	0.65 $\pm$ 0.065 &	0.705 $\pm$ 0.053 &	{\bf 0.726 $\pm$ 0.06 } &	0.652 $\pm$ 0.064 &	0.614 $\pm$ 0.061 \\
Haberman &	0.511 $\pm$ 0.019 &	0.5 $\pm$ 0.0 &	0.509 $\pm$ 0.012 &	0.536 $\pm$ 0.032 &	0.528 $\pm$ 0.048 &	0.519 $\pm$ 0.033 &	{\bf 0.569 $\pm$ 0.046 } \\
Segment &	0.985 $\pm$ 0.006 &	0.714 $\pm$ 0.031 &	0.808 $\pm$ 0.043 &	0.942 $\pm$ 0.025 &	0.921 $\pm$ 0.024 &	0.915 $\pm$ 0.031 &	{\bf 0.968 $\pm$ 0.012 } \\
Waveform &	0.839 $\pm$ 0.014 &	0.65 $\pm$ 0.012 &	0.717 $\pm$ 0.021 &	0.845 $\pm$ 0.014 &	0.764 $\pm$ 0.02 &	0.782 $\pm$ 0.016 &	{\bf 0.85 $\pm$ 0.012 } \\
Yeast &	0.518 $\pm$ 0.01 &	0.501 $\pm$ 0.002 &	0.516 $\pm$ 0.007 &	0.619 $\pm$ 0.017 &	0.654 $\pm$ 0.024 &	0.625 $\pm$ 0.012 &	{\bf 0.675 $\pm$ 0.014 } \\
Musk &	0.764 $\pm$ 0.014 &	0.648 $\pm$ 0.018 &	0.684 $\pm$ 0.022 &	0.845 $\pm$ 0.018 &	0.754 $\pm$ 0.018 &	0.774 $\pm$ 0.022 &	{\bf 0.852 $\pm$ 0.016 } \\
Isolet &	0.603 $\pm$ 0.023 &	0.542 $\pm$ 0.016 &	0.587 $\pm$ 0.021 &	{\bf 0.892 $\pm$ 0.015 } &	0.652 $\pm$ 0.028 &	0.807 $\pm$ 0.035 &	0.871 $\pm$ 0.026 \\
Semeion &	0.787 $\pm$ 0.041 &	0.69 $\pm$ 0.033 &	0.734 $\pm$ 0.047 &	0.806 $\pm$ 0.024 &	0.778 $\pm$ 0.04 &	0.793 $\pm$ 0.026 &	{\bf 0.817 $\pm$ 0.035 } \\
Vehicle &	0.937 $\pm$ 0.017 &	0.677 $\pm$ 0.042 &	0.778 $\pm$ 0.052 &	0.863 $\pm$ 0.028 &	0.834 $\pm$ 0.063 &	0.777 $\pm$ 0.04 &	{\bf 0.919 $\pm$ 0.038 } \\
CIFAR10 &	0.893 $\pm$ 0.0 &	0.735 $\pm$ 0.003 &	0.778 $\pm$ 0.015 &	0.858 $\pm$ 0.007 &	0.807 $\pm$ 0.005 &	0.786 $\pm$ 0.01 &	{\bf 0.886 $\pm$ 0.007 } \\
MNIST &	0.853 $\pm$ 0.0 &	0.699 $\pm$ 0.003 &	0.733 $\pm$ 0.005 &	{\bf 0.816 $\pm$ 0.009 } &	0.796 $\pm$ 0.002 &	0.758 $\pm$ 0.006 &	0.788 $\pm$ 0.018 \\
USPS &	0.879 $\pm$ 0.0 &	0.72 $\pm$ 0.005 &	0.761 $\pm$ 0.011 &	{\bf 0.85 $\pm$ 0.009 } &	0.8 $\pm$ 0.003 &	0.743 $\pm$ 0.003 &	0.825 $\pm$ 0.016 \\
Fashion &	0.948 $\pm$ 0.0 &	0.747 $\pm$ 0.003 &	0.768 $\pm$ 0.007 &	0.895 $\pm$ 0.008 &	0.824 $\pm$ 0.006 &	0.779 $\pm$ 0.007 &	{\bf 0.908 $\pm$ 0.005 } \\
\midrule
Avg. rank &	0.0 &	1.0 &	2.25 &	4.85 &	4.35 &	3.3 &	5.25 \\
\bottomrule
\end{tabular}
}
\end{table}

\begin{table}[h!]
\centering
{\tiny
\caption{Balanced accuracy for  datasets, for labelling strategy {\bf S2} and $c=0.7$.}
\label{tab:S2c07}
% SCAR c=0.1
\begin{tabular}{l|l|llllll}
\toprule 
Dataset &	\textsc{ORACLE} &	\textsc{NAIVE} &	\textsc{SAR-EM}  &	\textsc{LBE}  & \textsc{PGLIN}  & \textsc{TM} & \textsc{JERM} \\
\midrule
Breast-w &	0.952 $\pm$ 0.017 &	0.843 $\pm$ 0.032 &	0.87 $\pm$ 0.02 &	0.921 $\pm$ 0.027 &	0.912 $\pm$ 0.019 &	0.86 $\pm$ 0.027 &	{\bf 0.923 $\pm$ 0.019 } \\
Diabetes &	0.716 $\pm$ 0.026 &	0.626 $\pm$ 0.018 &	0.701 $\pm$ 0.033 &	0.647 $\pm$ 0.026 &	{\bf 0.756 $\pm$ 0.018 } &	0.703 $\pm$ 0.033 &	0.685 $\pm$ 0.036 \\
Spambase &	0.893 $\pm$ 0.01 &	0.808 $\pm$ 0.018 &	0.829 $\pm$ 0.015 &	{\bf 0.887 $\pm$ 0.016 } &	0.867 $\pm$ 0.014 &	0.83 $\pm$ 0.014 &	0.881 $\pm$ 0.011 \\
Wdbc &	0.971 $\pm$ 0.01 &	0.86 $\pm$ 0.03 &	0.878 $\pm$ 0.032 &	0.88 $\pm$ 0.035 &	{\bf 0.912 $\pm$ 0.028 } &	0.847 $\pm$ 0.035 &	0.899 $\pm$ 0.039 \\
Banknote &	0.991 $\pm$ 0.005 &	0.879 $\pm$ 0.023 &	0.91 $\pm$ 0.027 &	0.965 $\pm$ 0.021 &	0.93 $\pm$ 0.022 &	0.882 $\pm$ 0.026 &	{\bf 0.981 $\pm$ 0.006 } \\
Heart &	0.808 $\pm$ 0.038 &	0.764 $\pm$ 0.039 &	0.786 $\pm$ 0.036 &	0.754 $\pm$ 0.032 &	{\bf 0.805 $\pm$ 0.04 } &	0.767 $\pm$ 0.031 &	0.743 $\pm$ 0.061 \\
Ionosphere &	0.82 $\pm$ 0.022 &	0.798 $\pm$ 0.03 &	0.805 $\pm$ 0.029 &	0.754 $\pm$ 0.029 &	{\bf 0.827 $\pm$ 0.018 } &	0.814 $\pm$ 0.038 &	0.783 $\pm$ 0.049 \\
Sonar &	0.782 $\pm$ 0.049 &	0.644 $\pm$ 0.04 &	0.705 $\pm$ 0.038 &	{\bf 0.733 $\pm$ 0.061 } &	0.688 $\pm$ 0.05 &	0.728 $\pm$ 0.081 &	0.659 $\pm$ 0.067 \\
Haberman &	0.511 $\pm$ 0.019 &	0.505 $\pm$ 0.012 &	0.526 $\pm$ 0.028 &	0.523 $\pm$ 0.036 &	0.528 $\pm$ 0.051 &	0.513 $\pm$ 0.022 &	{\bf 0.585 $\pm$ 0.029 } \\
Segment &	0.985 $\pm$ 0.006 &	0.881 $\pm$ 0.033 &	0.907 $\pm$ 0.025 &	0.967 $\pm$ 0.013 &	0.969 $\pm$ 0.01 &	0.952 $\pm$ 0.016 &	{\bf 0.98 $\pm$ 0.007 } \\
Waveform &	0.839 $\pm$ 0.014 &	0.732 $\pm$ 0.013 &	0.781 $\pm$ 0.015 &	{\bf 0.849 $\pm$ 0.007 } &	0.831 $\pm$ 0.019 &	0.806 $\pm$ 0.012 &	0.839 $\pm$ 0.009 \\
Yeast &	0.518 $\pm$ 0.01 &	0.506 $\pm$ 0.006 &	0.545 $\pm$ 0.018 &	0.617 $\pm$ 0.015 &	{\bf 0.658 $\pm$ 0.029 } &	0.63 $\pm$ 0.019 &	0.652 $\pm$ 0.011 \\
Musk &	0.764 $\pm$ 0.014 &	0.722 $\pm$ 0.015 &	0.754 $\pm$ 0.017 &	0.836 $\pm$ 0.009 &	0.796 $\pm$ 0.014 &	0.789 $\pm$ 0.016 &	{\bf 0.863 $\pm$ 0.01 } \\
Isolet &	0.603 $\pm$ 0.023 &	0.566 $\pm$ 0.023 &	0.628 $\pm$ 0.028 &	{\bf 0.905 $\pm$ 0.013 } &	0.715 $\pm$ 0.034 &	0.821 $\pm$ 0.029 &	0.889 $\pm$ 0.022 \\
Semeion &	0.787 $\pm$ 0.041 &	0.733 $\pm$ 0.037 &	0.783 $\pm$ 0.035 &	0.837 $\pm$ 0.037 &	0.82 $\pm$ 0.03 &	0.809 $\pm$ 0.033 &	{\bf 0.849 $\pm$ 0.034 } \\
Vehicle &	0.937 $\pm$ 0.017 &	0.811 $\pm$ 0.043 &	0.875 $\pm$ 0.039 &	{\bf 0.933 $\pm$ 0.03 } &	0.916 $\pm$ 0.022 &	0.87 $\pm$ 0.042 &	0.911 $\pm$ 0.036 \\
CIFAR10 &	0.893 $\pm$ 0.0 &	0.828 $\pm$ 0.003 &	0.853 $\pm$ 0.005 &	0.871 $\pm$ 0.005 &	0.874 $\pm$ 0.002 &	0.845 $\pm$ 0.002 &	{\bf 0.885 $\pm$ 0.003 } \\
MNIST &	0.853 $\pm$ 0.0 &	0.783 $\pm$ 0.002 &	0.814 $\pm$ 0.004 &	0.807 $\pm$ 0.007 &	{\bf 0.843 $\pm$ 0.003 } &	0.809 $\pm$ 0.002 &	0.753 $\pm$ 0.033 \\
USPS &	0.879 $\pm$ 0.0 &	0.812 $\pm$ 0.002 &	0.831 $\pm$ 0.009 &	0.838 $\pm$ 0.011 &	{\bf 0.861 $\pm$ 0.003 } &	0.812 $\pm$ 0.003 &	0.84 $\pm$ 0.013 \\
Fashion &	0.948 $\pm$ 0.0 &	0.849 $\pm$ 0.002 &	0.87 $\pm$ 0.005 &	0.918 $\pm$ 0.002 &	0.902 $\pm$ 0.004 &	0.854 $\pm$ 0.001 &	{\bf 0.925 $\pm$ 0.01 } \\
\midrule
Avg. rank &	0.0 &	1.3 &	3.05 &	4.25 &	4.8 &	3.05 &	4.55 \\
\bottomrule
\end{tabular}
}
\end{table}

%%%%%%%%%%%%%%%%%%%%%%%%%%%%%%%%%%%%%%%%%%%%%%%%%%%%%%%%%%%%%%%%%%%%%%%%%%%%%%%%%%%%%%%%%%%%%%%%%%%%%%%%%%%%%%%%%%%%%%%%%%%%
%%%%%%%%%%%%%%%%%%%%%%%%%%%%%%%%%%%%%%%%%%%%%%%%%%%%%%%%%%%%%%%%%%%%%%%%%%%%%%%%%%%%%%%%%%%%%%%%%%%%%%%%%%%%%%%%%%%%%%%%%%%%
%%%%%%%%%%%%%%%%%%%%%%%%%%%%%%%%%%%%%%%%%%%%%%%%%%%%%%%%%%%%%%%%%%%%%%%%%%%%%%%%%%%%%%%%%%%%%%%%%%%%%%%%%%%%%%%%%%%%%%%%%%%%
\clearpage

\subsection{Results for labeling strategy S3}

\begin{table}[h!]
\centering
{\tiny
\caption{Balanced accuracy for  datasets, for labelling strategy {\bf S3} and $c=0.3$.}
\label{tab:S3c03}
% SCAR c=0.1
\begin{tabular}{l|l|llllll}
\toprule 
Dataset &	\textsc{ORACLE} &	\textsc{NAIVE} &	\textsc{SAR-EM}  &	\textsc{LBE}  & \textsc{PGLIN}  & \textsc{TM} & \textsc{JERM} \\
\midrule
Breast-w &	0.952 $\pm$ 0.017 &	0.578 $\pm$ 0.022 &	0.645 $\pm$ 0.057 &	{\bf 0.868 $\pm$ 0.04 } &	0.725 $\pm$ 0.049 &	0.769 $\pm$ 0.037 &	0.846 $\pm$ 0.044 \\
Diabetes &	0.716 $\pm$ 0.026 &	0.504 $\pm$ 0.004 &	0.549 $\pm$ 0.028 &	0.675 $\pm$ 0.054 &	{\bf 0.705 $\pm$ 0.039 } &	0.665 $\pm$ 0.041 &	0.697 $\pm$ 0.034 \\
Spambase &	0.892 $\pm$ 0.01 &	0.619 $\pm$ 0.011 &	0.645 $\pm$ 0.029 &	{\bf 0.86 $\pm$ 0.019 } &	0.683 $\pm$ 0.012 &	0.709 $\pm$ 0.023 &	0.834 $\pm$ 0.064 \\
Wdbc &	0.971 $\pm$ 0.01 &	0.656 $\pm$ 0.026 &	0.688 $\pm$ 0.03 &	0.77 $\pm$ 0.04 &	0.722 $\pm$ 0.029 &	0.707 $\pm$ 0.027 &	{\bf 0.77 $\pm$ 0.042 } \\
Banknote &	0.991 $\pm$ 0.005 &	0.61 $\pm$ 0.021 &	0.707 $\pm$ 0.023 &	0.959 $\pm$ 0.024 &	0.742 $\pm$ 0.023 &	0.815 $\pm$ 0.053 &	{\bf 0.982 $\pm$ 0.011 } \\
Heart &	0.808 $\pm$ 0.038 &	0.554 $\pm$ 0.022 &	0.616 $\pm$ 0.038 &	0.713 $\pm$ 0.036 &	{\bf 0.726 $\pm$ 0.054 } &	0.69 $\pm$ 0.048 &	0.7 $\pm$ 0.088 \\
Ionosphere &	0.82 $\pm$ 0.022 &	0.527 $\pm$ 0.031 &	0.547 $\pm$ 0.023 &	0.739 $\pm$ 0.041 &	{\bf 0.801 $\pm$ 0.038 } &	0.717 $\pm$ 0.063 &	0.667 $\pm$ 0.061 \\
Sonar &	0.782 $\pm$ 0.05 &	0.507 $\pm$ 0.012 &	0.548 $\pm$ 0.038 &	0.66 $\pm$ 0.049 &	{\bf 0.73 $\pm$ 0.067 } &	0.67 $\pm$ 0.047 &	0.601 $\pm$ 0.053 \\
Haberman &	0.511 $\pm$ 0.019 &	0.5 $\pm$ 0.0 &	0.509 $\pm$ 0.017 &	0.583 $\pm$ 0.064 &	0.535 $\pm$ 0.042 &	0.506 $\pm$ 0.033 &	{\bf 0.607 $\pm$ 0.065 } \\
Segment &	0.985 $\pm$ 0.006 &	0.544 $\pm$ 0.024 &	0.671 $\pm$ 0.049 &	0.931 $\pm$ 0.038 &	0.78 $\pm$ 0.042 &	0.889 $\pm$ 0.025 &	{\bf 0.953 $\pm$ 0.029 } \\
Waveform &	0.837 $\pm$ 0.016 &	0.557 $\pm$ 0.007 &	0.668 $\pm$ 0.027 &	0.823 $\pm$ 0.009 &	0.666 $\pm$ 0.014 &	0.788 $\pm$ 0.019 &	{\bf 0.844 $\pm$ 0.032 } \\
Yeast &	0.518 $\pm$ 0.01 &	0.5 $\pm$ 0.0 &	0.504 $\pm$ 0.005 &	0.646 $\pm$ 0.019 &	0.608 $\pm$ 0.046 &	0.619 $\pm$ 0.022 &	{\bf 0.672 $\pm$ 0.017 } \\
Musk &	0.753 $\pm$ 0.014 &	0.574 $\pm$ 0.012 &	0.608 $\pm$ 0.019 &	{\bf 0.839 $\pm$ 0.016 } &	0.632 $\pm$ 0.02 &	0.783 $\pm$ 0.022 &	0.828 $\pm$ 0.014 \\
Isolet &	0.625 $\pm$ 0.03 &	0.506 $\pm$ 0.007 &	0.535 $\pm$ 0.02 &	{\bf 0.855 $\pm$ 0.036 } &	0.622 $\pm$ 0.026 &	0.81 $\pm$ 0.039 &	0.844 $\pm$ 0.027 \\
Semeion &	0.804 $\pm$ 0.031 &	0.507 $\pm$ 0.015 &	0.594 $\pm$ 0.059 &	0.793 $\pm$ 0.055 &	0.782 $\pm$ 0.043 &	{\bf 0.805 $\pm$ 0.041 } &	0.8 $\pm$ 0.049 \\
Vehicle &	0.937 $\pm$ 0.017 &	0.572 $\pm$ 0.032 &	0.676 $\pm$ 0.056 &	0.845 $\pm$ 0.041 &	0.696 $\pm$ 0.038 &	0.789 $\pm$ 0.036 &	{\bf 0.859 $\pm$ 0.041 } \\
CIFAR10 &	0.893 $\pm$ 0.0 &	0.604 $\pm$ 0.006 &	0.694 $\pm$ 0.023 &	0.833 $\pm$ 0.006 &	0.709 $\pm$ 0.008 &	0.793 $\pm$ 0.013 &	{\bf 0.865 $\pm$ 0.009 } \\
MNIST &	0.851 $\pm$ 0.0 &	0.594 $\pm$ 0.004 &	0.649 $\pm$ 0.015 &	{\bf 0.798 $\pm$ 0.009 } &	0.681 $\pm$ 0.004 &	0.745 $\pm$ 0.011 &	0.76 $\pm$ 0.018 \\
USPS &	0.884 $\pm$ 0.0 &	0.599 $\pm$ 0.006 &	0.668 $\pm$ 0.028 &	{\bf 0.825 $\pm$ 0.01 } &	0.708 $\pm$ 0.003 &	0.738 $\pm$ 0.005 &	0.75 $\pm$ 0.024 \\
Fashion &	0.945 $\pm$ 0.0 &	0.613 $\pm$ 0.004 &	0.7 $\pm$ 0.008 &	0.884 $\pm$ 0.008 &	0.718 $\pm$ 0.008 &	0.775 $\pm$ 0.008 &	{\bf 0.892 $\pm$ 0.017 } \\
\midrule
Avg. rank &	0.0 &	1.0 &	2.1 &	5.15 &	3.65 &	3.9 &	5.2 \\
\bottomrule
\end{tabular}
}
\end{table}

\begin{table}[h!]
\centering
{\tiny
\caption{Balanced accuracy for  datasets, for labelling strategy {\bf S3} and $c=0.5$.}
\label{tab:S3c05}
% SCAR c=0.1
\begin{tabular}{l|l|llllll}
\toprule 
Dataset &	\textsc{ORACLE} &	\textsc{NAIVE} &	\textsc{SAR-EM}  &	\textsc{LBE}  & \textsc{PGLIN}  & \textsc{TM} & \textsc{JERM} \\
\midrule
Breast-w &	0.952 $\pm$ 0.017 &	0.719 $\pm$ 0.042 &	0.777 $\pm$ 0.036 &	0.886 $\pm$ 0.028 &	0.846 $\pm$ 0.036 &	0.8 $\pm$ 0.04 &	{\bf 0.905 $\pm$ 0.039 } \\
Diabetes &	0.716 $\pm$ 0.026 &	0.565 $\pm$ 0.014 &	0.628 $\pm$ 0.023 &	0.655 $\pm$ 0.029 &	{\bf 0.72 $\pm$ 0.037 } &	0.692 $\pm$ 0.034 &	0.706 $\pm$ 0.031 \\
Spambase &	0.892 $\pm$ 0.01 &	0.718 $\pm$ 0.014 &	0.736 $\pm$ 0.018 &	{\bf 0.875 $\pm$ 0.014 } &	0.777 $\pm$ 0.017 &	0.75 $\pm$ 0.013 &	0.864 $\pm$ 0.086 \\
Wdbc &	0.971 $\pm$ 0.01 &	0.761 $\pm$ 0.033 &	0.782 $\pm$ 0.035 &	0.838 $\pm$ 0.032 &	0.813 $\pm$ 0.041 &	0.795 $\pm$ 0.031 &	{\bf 0.846 $\pm$ 0.024 } \\
Banknote &	0.991 $\pm$ 0.005 &	0.753 $\pm$ 0.022 &	0.829 $\pm$ 0.049 &	0.979 $\pm$ 0.01 &	0.862 $\pm$ 0.023 &	0.795 $\pm$ 0.025 &	{\bf 0.988 $\pm$ 0.005 } \\
Heart &	0.808 $\pm$ 0.038 &	0.677 $\pm$ 0.042 &	0.719 $\pm$ 0.056 &	0.758 $\pm$ 0.037 &	{\bf 0.784 $\pm$ 0.044 } &	0.731 $\pm$ 0.062 &	0.713 $\pm$ 0.09 \\
Ionosphere &	0.82 $\pm$ 0.022 &	0.624 $\pm$ 0.031 &	0.696 $\pm$ 0.059 &	0.741 $\pm$ 0.043 &	{\bf 0.824 $\pm$ 0.03 } &	0.774 $\pm$ 0.03 &	0.721 $\pm$ 0.053 \\
Sonar &	0.782 $\pm$ 0.05 &	0.551 $\pm$ 0.049 &	0.641 $\pm$ 0.034 &	0.715 $\pm$ 0.044 &	{\bf 0.746 $\pm$ 0.054 } &	0.665 $\pm$ 0.054 &	0.624 $\pm$ 0.064 \\
Haberman &	0.511 $\pm$ 0.019 &	0.503 $\pm$ 0.009 &	0.514 $\pm$ 0.018 &	0.547 $\pm$ 0.027 &	0.523 $\pm$ 0.03 &	0.5 $\pm$ 0.019 &	{\bf 0.551 $\pm$ 0.035 } \\
Segment &	0.985 $\pm$ 0.006 &	0.71 $\pm$ 0.02 &	0.819 $\pm$ 0.033 &	0.955 $\pm$ 0.019 &	0.914 $\pm$ 0.027 &	0.899 $\pm$ 0.02 &	{\bf 0.97 $\pm$ 0.013 } \\
Waveform &	0.837 $\pm$ 0.016 &	0.65 $\pm$ 0.014 &	0.727 $\pm$ 0.019 &	0.842 $\pm$ 0.013 &	0.763 $\pm$ 0.018 &	0.782 $\pm$ 0.018 &	{\bf 0.845 $\pm$ 0.019 } \\
Yeast &	0.518 $\pm$ 0.01 &	0.5 $\pm$ 0.001 &	0.516 $\pm$ 0.007 &	0.621 $\pm$ 0.015 &	{\bf 0.662 $\pm$ 0.024 } &	0.624 $\pm$ 0.018 &	0.66 $\pm$ 0.019 \\
Musk &	0.753 $\pm$ 0.014 &	0.63 $\pm$ 0.017 &	0.672 $\pm$ 0.029 &	{\bf 0.847 $\pm$ 0.014 } &	0.741 $\pm$ 0.02 &	0.78 $\pm$ 0.017 &	0.837 $\pm$ 0.019 \\
Isolet &	0.625 $\pm$ 0.03 &	0.532 $\pm$ 0.01 &	0.593 $\pm$ 0.019 &	{\bf 0.896 $\pm$ 0.02 } &	0.685 $\pm$ 0.025 &	0.843 $\pm$ 0.033 &	0.88 $\pm$ 0.022 \\
Semeion &	0.804 $\pm$ 0.031 &	0.649 $\pm$ 0.044 &	0.712 $\pm$ 0.065 &	0.807 $\pm$ 0.035 &	0.798 $\pm$ 0.028 &	0.81 $\pm$ 0.03 &	{\bf 0.829 $\pm$ 0.036 } \\
Vehicle &	0.937 $\pm$ 0.017 &	0.675 $\pm$ 0.027 &	0.76 $\pm$ 0.038 &	0.875 $\pm$ 0.034 &	0.839 $\pm$ 0.052 &	0.776 $\pm$ 0.053 &	{\bf 0.88 $\pm$ 0.053 } \\
CIFAR10 &	0.893 $\pm$ 0.0 &	0.738 $\pm$ 0.009 &	0.782 $\pm$ 0.015 &	0.845 $\pm$ 0.003 &	0.824 $\pm$ 0.004 &	0.797 $\pm$ 0.008 &	{\bf 0.883 $\pm$ 0.005 } \\
MNIST &	0.851 $\pm$ 0.0 &	0.694 $\pm$ 0.003 &	0.738 $\pm$ 0.015 &	{\bf 0.801 $\pm$ 0.009 } &	0.797 $\pm$ 0.002 &	0.754 $\pm$ 0.004 &	0.791 $\pm$ 0.024 \\
USPS &	0.884 $\pm$ 0.0 &	0.717 $\pm$ 0.004 &	0.745 $\pm$ 0.019 &	{\bf 0.829 $\pm$ 0.008 } &	0.811 $\pm$ 0.004 &	0.748 $\pm$ 0.004 &	0.786 $\pm$ 0.022 \\
Fashion &	0.945 $\pm$ 0.0 &	0.741 $\pm$ 0.004 &	0.764 $\pm$ 0.012 &	0.889 $\pm$ 0.007 &	0.836 $\pm$ 0.004 &	0.777 $\pm$ 0.004 &	{\bf 0.916 $\pm$ 0.007 } \\
\midrule
Avg. rank &	0.0 &	1.05 &	2.2 &	4.95 &	4.4 &	3.4 &	5.0 \\
\bottomrule
\end{tabular}
}
\end{table}

\begin{table}[h!]
\centering
{\tiny
\caption{Balanced accuracy for datasets, for labelling strategy {\bf S3} and $c=0.7$.}
\label{tab:S3c07}
% SCAR c=0.1
\begin{tabular}{l|l|llllll}
\toprule 
Dataset &	\textsc{ORACLE} &	\textsc{NAIVE} &	\textsc{SAR-EM}  &	\textsc{LBE}  & \textsc{PGLIN}  & \textsc{TM} & \textsc{JERM} \\
\midrule
Breast-w &	0.952 $\pm$ 0.017 &	0.844 $\pm$ 0.024 &	0.887 $\pm$ 0.016 &	0.912 $\pm$ 0.029 &	0.919 $\pm$ 0.023 &	0.876 $\pm$ 0.028 &	{\bf 0.942 $\pm$ 0.019 } \\
Diabetes &	0.716 $\pm$ 0.026 &	0.62 $\pm$ 0.026 &	0.696 $\pm$ 0.032 &	0.656 $\pm$ 0.018 &	{\bf 0.754 $\pm$ 0.02 } &	0.708 $\pm$ 0.033 &	0.692 $\pm$ 0.024 \\
Spambase &	0.892 $\pm$ 0.01 &	0.794 $\pm$ 0.015 &	0.818 $\pm$ 0.015 &	0.844 $\pm$ 0.115 &	{\bf 0.872 $\pm$ 0.013 } &	0.841 $\pm$ 0.016 &	0.842 $\pm$ 0.095 \\
Wdbc &	0.971 $\pm$ 0.01 &	0.86 $\pm$ 0.028 &	0.875 $\pm$ 0.026 &	{\bf 0.909 $\pm$ 0.037 } &	0.906 $\pm$ 0.029 &	0.867 $\pm$ 0.021 &	0.875 $\pm$ 0.043 \\
Banknote &	0.991 $\pm$ 0.005 &	0.887 $\pm$ 0.03 &	0.906 $\pm$ 0.023 &	0.971 $\pm$ 0.02 &	0.954 $\pm$ 0.013 &	0.896 $\pm$ 0.021 &	{\bf 0.982 $\pm$ 0.005 } \\
Heart &	0.808 $\pm$ 0.038 &	0.771 $\pm$ 0.035 &	0.789 $\pm$ 0.042 &	0.749 $\pm$ 0.056 &	{\bf 0.805 $\pm$ 0.05 } &	0.778 $\pm$ 0.031 &	0.768 $\pm$ 0.079 \\
Ionosphere &	0.82 $\pm$ 0.022 &	0.802 $\pm$ 0.027 &	0.791 $\pm$ 0.036 &	0.748 $\pm$ 0.037 &	0.821 $\pm$ 0.02 &	{\bf 0.826 $\pm$ 0.035 } &	0.776 $\pm$ 0.032 \\
Sonar &	0.782 $\pm$ 0.05 &	0.648 $\pm$ 0.073 &	0.711 $\pm$ 0.055 &	{\bf 0.743 $\pm$ 0.056 } &	0.681 $\pm$ 0.057 &	0.712 $\pm$ 0.047 &	0.682 $\pm$ 0.065 \\
Haberman &	0.511 $\pm$ 0.019 &	0.507 $\pm$ 0.013 &	0.527 $\pm$ 0.027 &	0.537 $\pm$ 0.041 &	0.527 $\pm$ 0.038 &	0.506 $\pm$ 0.031 &	{\bf 0.542 $\pm$ 0.039 } \\
Segment &	0.985 $\pm$ 0.006 &	0.886 $\pm$ 0.048 &	0.916 $\pm$ 0.026 &	0.965 $\pm$ 0.019 &	0.972 $\pm$ 0.009 &	0.947 $\pm$ 0.02 &	{\bf 0.978 $\pm$ 0.008 } \\
Waveform &	0.837 $\pm$ 0.016 &	0.729 $\pm$ 0.019 &	0.782 $\pm$ 0.02 &	{\bf 0.852 $\pm$ 0.008 } &	0.839 $\pm$ 0.015 &	0.818 $\pm$ 0.021 &	0.824 $\pm$ 0.016 \\
Yeast &	0.518 $\pm$ 0.01 &	0.505 $\pm$ 0.005 &	0.544 $\pm$ 0.016 &	0.629 $\pm$ 0.015 &	0.644 $\pm$ 0.032 &	0.637 $\pm$ 0.013 &	{\bf 0.656 $\pm$ 0.011 } \\
Musk &	0.753 $\pm$ 0.014 &	0.691 $\pm$ 0.023 &	0.725 $\pm$ 0.023 &	0.842 $\pm$ 0.011 &	0.787 $\pm$ 0.016 &	0.789 $\pm$ 0.018 &	{\bf 0.856 $\pm$ 0.01 } \\
Isolet &	0.625 $\pm$ 0.03 &	0.576 $\pm$ 0.021 &	0.641 $\pm$ 0.029 &	{\bf 0.909 $\pm$ 0.015 } &	0.733 $\pm$ 0.018 &	0.834 $\pm$ 0.024 &	0.899 $\pm$ 0.011 \\
Semeion &	0.804 $\pm$ 0.031 &	0.772 $\pm$ 0.046 &	0.797 $\pm$ 0.039 &	0.859 $\pm$ 0.035 &	0.82 $\pm$ 0.025 &	0.817 $\pm$ 0.027 &	{\bf 0.873 $\pm$ 0.024 } \\
Vehicle &	0.937 $\pm$ 0.017 &	0.798 $\pm$ 0.044 &	0.881 $\pm$ 0.031 &	{\bf 0.934 $\pm$ 0.017 } &	0.925 $\pm$ 0.023 &	0.853 $\pm$ 0.044 &	0.902 $\pm$ 0.031 \\
CIFAR10 &	0.893 $\pm$ 0.0 &	0.825 $\pm$ 0.003 &	0.846 $\pm$ 0.005 &	0.855 $\pm$ 0.015 &	{\bf 0.882 $\pm$ 0.001 } &	0.853 $\pm$ 0.005 &	0.873 $\pm$ 0.006 \\
MNIST &	0.851 $\pm$ 0.0 &	0.79 $\pm$ 0.002 &	0.81 $\pm$ 0.005 &	0.796 $\pm$ 0.011 &	{\bf 0.846 $\pm$ 0.002 } &	0.813 $\pm$ 0.002 &	0.797 $\pm$ 0.024 \\
USPS &	0.884 $\pm$ 0.0 &	0.812 $\pm$ 0.005 &	0.83 $\pm$ 0.007 &	0.834 $\pm$ 0.008 &	{\bf 0.864 $\pm$ 0.002 } &	0.826 $\pm$ 0.005 &	0.816 $\pm$ 0.016 \\
Fashion &	0.945 $\pm$ 0.0 &	0.849 $\pm$ 0.003 &	0.869 $\pm$ 0.008 &	0.916 $\pm$ 0.004 &	0.913 $\pm$ 0.002 &	0.868 $\pm$ 0.002 &	{\bf 0.929 $\pm$ 0.012 } \\
\midrule
Avg. rank &	0.0 &	1.3 &	2.9 &	4.3 &	4.75 &	3.3 &	4.45 \\
\bottomrule
\end{tabular}
}
\end{table}

%%%%%%%%%%%%%%%%%%%%%%%%%%%%%%%%%%%%%%%%%%%%%%%%%%%%%%%%%%%%%%%%%%%%%%%%%%%%%%%%%%%%%%%%%%%%%%%%%%%%%%%%%%%%%%%%%%%%%%%%%%%%
%%%%%%%%%%%%%%%%%%%%%%%%%%%%%%%%%%%%%%%%%%%%%%%%%%%%%%%%%%%%%%%%%%%%%%%%%%%%%%%%%%%%%%%%%%%%%%%%%%%%%%%%%%%%%%%%%%%%%%%%%%%%
%%%%%%%%%%%%%%%%%%%%%%%%%%%%%%%%%%%%%%%%%%%%%%%%%%%%%%%%%%%%%%%%%%%%%%%%%%%%%%%%%%%%%%%%%%%%%%%%%%%%%%%%%%%%%%%%%%%%%%%%%%%%
\clearpage

\subsection{Results for labeling strategy S4}

\begin{table}[h!]
\centering
{\tiny
\caption{balanced accuracy for datasets, for labelling strategy {\bf S4} and $c=0.3$.}
\label{tab:S4c03}
% SCAR c=0.1
\begin{tabular}{l|l|llllll}
\toprule 
Dataset &	\textsc{ORACLE} &	\textsc{NAIVE} &	\textsc{SAR-EM}  &	\textsc{LBE}  & \textsc{PGLIN}  & \textsc{TM} & \textsc{JERM} \\
\midrule
Breast-w &	0.952 $\pm$ 0.017 &	0.599 $\pm$ 0.016 &	0.638 $\pm$ 0.033 &	0.75 $\pm$ 0.038 &	0.723 $\pm$ 0.037 &	0.659 $\pm$ 0.043 &	{\bf 0.783 $\pm$ 0.031 } \\
Diabetes &	0.716 $\pm$ 0.026 &	0.511 $\pm$ 0.009 &	0.566 $\pm$ 0.009 &	0.722 $\pm$ 0.027 &	0.683 $\pm$ 0.041 &	0.639 $\pm$ 0.029 &	{\bf 0.726 $\pm$ 0.023 } \\
Spambase &	0.893 $\pm$ 0.01 &	0.64 $\pm$ 0.011 &	0.645 $\pm$ 0.011 &	0.677 $\pm$ 0.06 &	0.67 $\pm$ 0.013 &	0.658 $\pm$ 0.014 &	{\bf 0.763 $\pm$ 0.029 } \\
Wdbc &	0.971 $\pm$ 0.01 &	0.648 $\pm$ 0.032 &	0.667 $\pm$ 0.04 &	0.678 $\pm$ 0.044 &	0.692 $\pm$ 0.039 &	0.67 $\pm$ 0.034 &	{\bf 0.749 $\pm$ 0.038 } \\
Banknote &	0.991 $\pm$ 0.005 &	0.64 $\pm$ 0.017 &	0.665 $\pm$ 0.024 &	0.767 $\pm$ 0.025 &	0.713 $\pm$ 0.02 &	0.676 $\pm$ 0.02 &	{\bf 0.961 $\pm$ 0.017 } \\
Heart &	0.808 $\pm$ 0.038 &	0.599 $\pm$ 0.037 &	0.613 $\pm$ 0.033 &	0.705 $\pm$ 0.058 &	{\bf 0.721 $\pm$ 0.056 } &	0.651 $\pm$ 0.05 &	0.679 $\pm$ 0.031 \\
Ionosphere &	0.82 $\pm$ 0.022 &	0.523 $\pm$ 0.029 &	0.547 $\pm$ 0.036 &	{\bf 0.756 $\pm$ 0.052 } &	0.743 $\pm$ 0.069 &	0.657 $\pm$ 0.058 &	0.697 $\pm$ 0.071 \\
Sonar &	0.782 $\pm$ 0.049 &	0.524 $\pm$ 0.031 &	0.56 $\pm$ 0.035 &	0.675 $\pm$ 0.057 &	{\bf 0.716 $\pm$ 0.07 } &	0.595 $\pm$ 0.074 &	0.597 $\pm$ 0.057 \\
Haberman &	0.511 $\pm$ 0.019 &	0.5 $\pm$ 0.0 &	0.512 $\pm$ 0.018 &	0.569 $\pm$ 0.072 &	0.534 $\pm$ 0.057 &	0.505 $\pm$ 0.023 &	{\bf 0.615 $\pm$ 0.068 } \\
Segment &	0.985 $\pm$ 0.006 &	0.579 $\pm$ 0.023 &	0.646 $\pm$ 0.033 &	0.842 $\pm$ 0.025 &	0.758 $\pm$ 0.039 &	0.812 $\pm$ 0.034 &	{\bf 0.88 $\pm$ 0.029 } \\
Waveform &	0.84 $\pm$ 0.015 &	0.582 $\pm$ 0.012 &	0.623 $\pm$ 0.015 &	0.776 $\pm$ 0.018 &	0.651 $\pm$ 0.014 &	0.719 $\pm$ 0.023 &	{\bf 0.822 $\pm$ 0.035 } \\
Yeast &	0.518 $\pm$ 0.01 &	0.5 $\pm$ 0.0 &	0.507 $\pm$ 0.006 &	0.662 $\pm$ 0.014 &	0.66 $\pm$ 0.041 &	0.628 $\pm$ 0.011 &	{\bf 0.693 $\pm$ 0.022 } \\
Musk &	0.754 $\pm$ 0.014 &	0.585 $\pm$ 0.012 &	0.628 $\pm$ 0.018 &	{\bf 0.778 $\pm$ 0.02 } &	0.652 $\pm$ 0.022 &	0.689 $\pm$ 0.021 &	0.742 $\pm$ 0.02 \\
Isolet &	0.602 $\pm$ 0.026 &	0.525 $\pm$ 0.008 &	0.553 $\pm$ 0.022 &	{\bf 0.803 $\pm$ 0.051 } &	0.6 $\pm$ 0.027 &	0.741 $\pm$ 0.029 &	0.756 $\pm$ 0.032 \\
Semeion &	0.786 $\pm$ 0.027 &	0.605 $\pm$ 0.067 &	0.667 $\pm$ 0.063 &	0.726 $\pm$ 0.037 &	0.751 $\pm$ 0.031 &	0.722 $\pm$ 0.041 &	{\bf 0.77 $\pm$ 0.044 } \\
Vehicle &	0.937 $\pm$ 0.017 &	0.598 $\pm$ 0.037 &	0.626 $\pm$ 0.03 &	0.753 $\pm$ 0.049 &	0.699 $\pm$ 0.033 &	0.666 $\pm$ 0.062 &	{\bf 0.799 $\pm$ 0.061 } \\
CIFAR10 &	0.895 $\pm$ 0.0 &	0.645 $\pm$ 0.003 &	0.674 $\pm$ 0.004 &	0.783 $\pm$ 0.01 &	0.701 $\pm$ 0.003 &	0.708 $\pm$ 0.005 &	{\bf 0.786 $\pm$ 0.01 } \\
MNIST &	0.848 $\pm$ 0.0 &	0.621 $\pm$ 0.004 &	0.648 $\pm$ 0.004 &	{\bf 0.778 $\pm$ 0.008 } &	0.677 $\pm$ 0.006 &	0.667 $\pm$ 0.004 &	0.77 $\pm$ 0.008 \\
USPS &	0.882 $\pm$ 0.0 &	0.624 $\pm$ 0.004 &	0.66 $\pm$ 0.003 &	0.769 $\pm$ 0.003 &	0.692 $\pm$ 0.004 &	0.676 $\pm$ 0.003 &	{\bf 0.78 $\pm$ 0.006 } \\
Fashion &	0.949 $\pm$ 0.0 &	0.641 $\pm$ 0.003 &	0.666 $\pm$ 0.006 &	0.779 $\pm$ 0.006 &	0.708 $\pm$ 0.002 &	0.688 $\pm$ 0.003 &	{\bf 0.805 $\pm$ 0.007 } \\
\midrule
Avg. rank &	0.0 &	1.0 &	2.05 &	5.1 &	4.1 &	3.2 &	5.55 \\
\bottomrule
\end{tabular}
}
\end{table}

\begin{table}[h!]
\centering
{\tiny
\caption{balanced accuracy for datasets, for labelling strategy {\bf S4} and $c=0.5$.}
\label{tab:S4c05}
% SCAR c=0.1
\begin{tabular}{l|l|llllll}
\toprule 
Dataset &	\textsc{ORACLE} &	\textsc{NAIVE} &	\textsc{SAR-EM}  &	\textsc{LBE}  & \textsc{PGLIN}  & \textsc{TM} & \textsc{JERM} \\
\midrule
Breast-w &	0.952 $\pm$ 0.017 &	0.733 $\pm$ 0.029 &	0.758 $\pm$ 0.036 &	0.848 $\pm$ 0.029 &	0.828 $\pm$ 0.028 &	0.766 $\pm$ 0.033 &	{\bf 0.859 $\pm$ 0.019 } \\
Diabetes &	0.716 $\pm$ 0.026 &	0.579 $\pm$ 0.017 &	0.637 $\pm$ 0.03 &	0.679 $\pm$ 0.039 &	{\bf 0.722 $\pm$ 0.031 } &	0.689 $\pm$ 0.034 &	0.714 $\pm$ 0.031 \\
Spambase &	0.893 $\pm$ 0.01 &	0.732 $\pm$ 0.014 &	0.743 $\pm$ 0.013 &	{\bf 0.816 $\pm$ 0.016 } &	0.77 $\pm$ 0.015 &	0.746 $\pm$ 0.015 &	0.799 $\pm$ 0.078 \\
Wdbc &	0.971 $\pm$ 0.01 &	0.76 $\pm$ 0.04 &	0.787 $\pm$ 0.046 &	0.806 $\pm$ 0.039 &	0.791 $\pm$ 0.043 &	0.762 $\pm$ 0.043 &	{\bf 0.852 $\pm$ 0.028 } \\
Banknote &	0.991 $\pm$ 0.005 &	0.75 $\pm$ 0.023 &	0.784 $\pm$ 0.021 &	0.86 $\pm$ 0.031 &	0.815 $\pm$ 0.019 &	0.768 $\pm$ 0.021 &	{\bf 0.986 $\pm$ 0.008 } \\
Heart &	0.808 $\pm$ 0.038 &	0.7 $\pm$ 0.042 &	0.712 $\pm$ 0.048 &	0.76 $\pm$ 0.049 &	{\bf 0.776 $\pm$ 0.028 } &	0.707 $\pm$ 0.061 &	0.727 $\pm$ 0.034 \\
Ionosphere &	0.82 $\pm$ 0.022 &	0.625 $\pm$ 0.047 &	0.663 $\pm$ 0.057 &	0.769 $\pm$ 0.056 &	{\bf 0.829 $\pm$ 0.024 } &	0.755 $\pm$ 0.037 &	0.738 $\pm$ 0.033 \\
Sonar &	0.782 $\pm$ 0.049 &	0.581 $\pm$ 0.059 &	0.649 $\pm$ 0.057 &	0.698 $\pm$ 0.041 &	{\bf 0.738 $\pm$ 0.058 } &	0.671 $\pm$ 0.059 &	0.649 $\pm$ 0.048 \\
Haberman &	0.511 $\pm$ 0.019 &	0.499 $\pm$ 0.003 &	0.509 $\pm$ 0.021 &	0.542 $\pm$ 0.027 &	0.56 $\pm$ 0.064 &	0.515 $\pm$ 0.024 &	{\bf 0.596 $\pm$ 0.061 } \\
Segment &	0.985 $\pm$ 0.006 &	0.707 $\pm$ 0.032 &	0.781 $\pm$ 0.033 &	0.903 $\pm$ 0.035 &	0.872 $\pm$ 0.031 &	0.88 $\pm$ 0.024 &	{\bf 0.964 $\pm$ 0.014 } \\
Waveform &	0.84 $\pm$ 0.015 &	0.658 $\pm$ 0.014 &	0.712 $\pm$ 0.015 &	0.835 $\pm$ 0.011 &	0.753 $\pm$ 0.013 &	0.767 $\pm$ 0.02 &	{\bf 0.854 $\pm$ 0.014 } \\
Yeast &	0.518 $\pm$ 0.01 &	0.5 $\pm$ 0.001 &	0.512 $\pm$ 0.008 &	0.637 $\pm$ 0.014 &	0.672 $\pm$ 0.026 &	0.632 $\pm$ 0.018 &	{\bf 0.681 $\pm$ 0.029 } \\
Musk &	0.754 $\pm$ 0.014 &	0.664 $\pm$ 0.024 &	0.693 $\pm$ 0.023 &	{\bf 0.842 $\pm$ 0.018 } &	0.746 $\pm$ 0.017 &	0.752 $\pm$ 0.02 &	0.839 $\pm$ 0.017 \\
Isolet &	0.602 $\pm$ 0.026 &	0.547 $\pm$ 0.017 &	0.578 $\pm$ 0.026 &	{\bf 0.864 $\pm$ 0.024 } &	0.645 $\pm$ 0.025 &	0.8 $\pm$ 0.044 &	0.847 $\pm$ 0.03 \\
Semeion &	0.786 $\pm$ 0.027 &	0.694 $\pm$ 0.033 &	0.723 $\pm$ 0.036 &	0.785 $\pm$ 0.04 &	0.773 $\pm$ 0.03 &	0.756 $\pm$ 0.039 &	{\bf 0.832 $\pm$ 0.035 } \\
Vehicle &	0.937 $\pm$ 0.017 &	0.697 $\pm$ 0.048 &	0.75 $\pm$ 0.056 &	0.86 $\pm$ 0.052 &	0.813 $\pm$ 0.047 &	0.764 $\pm$ 0.054 &	{\bf 0.895 $\pm$ 0.034 } \\
CIFAR10 &	0.895 $\pm$ 0.0 &	0.738 $\pm$ 0.002 &	0.771 $\pm$ 0.007 &	0.844 $\pm$ 0.005 &	0.796 $\pm$ 0.004 &	0.776 $\pm$ 0.003 &	{\bf 0.871 $\pm$ 0.004 } \\
MNIST &	0.848 $\pm$ 0.0 &	0.704 $\pm$ 0.004 &	0.742 $\pm$ 0.009 &	{\bf 0.827 $\pm$ 0.007 } &	0.786 $\pm$ 0.003 &	0.742 $\pm$ 0.002 &	0.822 $\pm$ 0.009 \\
USPS &	0.882 $\pm$ 0.0 &	0.724 $\pm$ 0.005 &	0.738 $\pm$ 0.011 &	0.833 $\pm$ 0.009 &	0.797 $\pm$ 0.003 &	0.737 $\pm$ 0.004 &	{\bf 0.84 $\pm$ 0.01 } \\
Fashion &	0.949 $\pm$ 0.0 &	0.749 $\pm$ 0.005 &	0.771 $\pm$ 0.007 &	0.862 $\pm$ 0.006 &	0.809 $\pm$ 0.004 &	0.764 $\pm$ 0.002 &	{\bf 0.878 $\pm$ 0.012 } \\
\midrule
Avg. rank &	0.0 &	1.0 &	2.25 &	5.0 &	4.3 &	3.1 &	5.35 \\
\bottomrule
\end{tabular}
}
\end{table}

\begin{table}[h!]
\centering
{\tiny
\caption{balanced accuracy for datasets, for labelling strategy {\bf S4} and $c=0.7$.}
\label{tab:S4c07}
% SCAR c=0.1
\begin{tabular}{l|l|llllll}
\toprule 
Dataset &	\textsc{ORACLE} &	\textsc{NAIVE} &	\textsc{SAR-EM}  &	\textsc{LBE}  & \textsc{PGLIN}  & \textsc{TM} & \textsc{JERM} \\
\midrule
Breast-w &	0.952 $\pm$ 0.017 &	0.852 $\pm$ 0.02 &	0.877 $\pm$ 0.022 &	0.907 $\pm$ 0.023 &	0.906 $\pm$ 0.016 &	0.865 $\pm$ 0.028 &	{\bf 0.917 $\pm$ 0.021 } \\
Diabetes &	0.716 $\pm$ 0.026 &	0.632 $\pm$ 0.027 &	0.698 $\pm$ 0.032 &	0.653 $\pm$ 0.031 &	{\bf 0.753 $\pm$ 0.026 } &	0.703 $\pm$ 0.031 &	0.691 $\pm$ 0.028 \\
Spambase &	0.893 $\pm$ 0.01 &	0.811 $\pm$ 0.016 &	0.822 $\pm$ 0.014 &	0.858 $\pm$ 0.12 &	{\bf 0.864 $\pm$ 0.014 } &	0.828 $\pm$ 0.017 &	0.799 $\pm$ 0.142 \\
Wdbc &	0.971 $\pm$ 0.01 &	0.865 $\pm$ 0.035 &	0.883 $\pm$ 0.029 &	0.886 $\pm$ 0.036 &	0.903 $\pm$ 0.025 &	0.865 $\pm$ 0.035 &	{\bf 0.904 $\pm$ 0.023 } \\
Banknote &	0.991 $\pm$ 0.005 &	0.871 $\pm$ 0.024 &	0.893 $\pm$ 0.023 &	0.938 $\pm$ 0.019 &	0.924 $\pm$ 0.02 &	0.876 $\pm$ 0.025 &	{\bf 0.98 $\pm$ 0.007 } \\
Heart &	0.808 $\pm$ 0.038 &	0.75 $\pm$ 0.035 &	0.795 $\pm$ 0.038 &	0.789 $\pm$ 0.048 &	{\bf 0.817 $\pm$ 0.029 } &	0.774 $\pm$ 0.035 &	0.77 $\pm$ 0.048 \\
Ionosphere &	0.82 $\pm$ 0.022 &	0.784 $\pm$ 0.034 &	0.817 $\pm$ 0.031 &	0.772 $\pm$ 0.034 &	{\bf 0.826 $\pm$ 0.036 } &	0.814 $\pm$ 0.046 &	0.776 $\pm$ 0.04 \\
Sonar &	0.782 $\pm$ 0.049 &	0.657 $\pm$ 0.065 &	0.739 $\pm$ 0.061 &	{\bf 0.749 $\pm$ 0.052 } &	0.705 $\pm$ 0.029 &	0.706 $\pm$ 0.031 &	0.675 $\pm$ 0.092 \\
Haberman &	0.511 $\pm$ 0.019 &	0.499 $\pm$ 0.003 &	0.527 $\pm$ 0.02 &	0.536 $\pm$ 0.02 &	0.538 $\pm$ 0.063 &	0.531 $\pm$ 0.044 &	{\bf 0.55 $\pm$ 0.027 } \\
Segment &	0.985 $\pm$ 0.006 &	0.885 $\pm$ 0.021 &	0.906 $\pm$ 0.024 &	0.947 $\pm$ 0.019 &	0.962 $\pm$ 0.016 &	0.932 $\pm$ 0.023 &	{\bf 0.98 $\pm$ 0.006 } \\
Waveform &	0.84 $\pm$ 0.015 &	0.735 $\pm$ 0.014 &	0.779 $\pm$ 0.014 &	{\bf 0.854 $\pm$ 0.007 } &	0.828 $\pm$ 0.012 &	0.803 $\pm$ 0.015 &	0.837 $\pm$ 0.009 \\
Yeast &	0.518 $\pm$ 0.01 &	0.506 $\pm$ 0.005 &	0.549 $\pm$ 0.018 &	0.628 $\pm$ 0.015 &	0.637 $\pm$ 0.024 &	0.637 $\pm$ 0.016 &	{\bf 0.662 $\pm$ 0.018 } \\
Musk &	0.754 $\pm$ 0.014 &	0.712 $\pm$ 0.015 &	0.746 $\pm$ 0.02 &	{\bf 0.869 $\pm$ 0.01 } &	0.788 $\pm$ 0.016 &	0.785 $\pm$ 0.017 &	0.865 $\pm$ 0.01 \\
Isolet &	0.602 $\pm$ 0.026 &	0.572 $\pm$ 0.018 &	0.629 $\pm$ 0.019 &	{\bf 0.899 $\pm$ 0.017 } &	0.696 $\pm$ 0.033 &	0.81 $\pm$ 0.035 &	0.882 $\pm$ 0.01 \\
Semeion &	0.786 $\pm$ 0.027 &	0.74 $\pm$ 0.039 &	0.771 $\pm$ 0.039 &	0.815 $\pm$ 0.033 &	0.801 $\pm$ 0.027 &	0.797 $\pm$ 0.031 &	{\bf 0.877 $\pm$ 0.029 } \\
Vehicle &	0.937 $\pm$ 0.017 &	0.807 $\pm$ 0.043 &	0.871 $\pm$ 0.034 &	0.917 $\pm$ 0.03 &	{\bf 0.92 $\pm$ 0.025 } &	0.872 $\pm$ 0.022 &	0.91 $\pm$ 0.024 \\
CIFAR10 &	0.895 $\pm$ 0.0 &	0.827 $\pm$ 0.001 &	0.845 $\pm$ 0.004 &	0.881 $\pm$ 0.004 &	0.869 $\pm$ 0.002 &	0.839 $\pm$ 0.002 &	{\bf 0.881 $\pm$ 0.017 } \\
MNIST &	0.848 $\pm$ 0.0 &	0.792 $\pm$ 0.003 &	0.821 $\pm$ 0.003 &	0.813 $\pm$ 0.005 &	{\bf 0.84 $\pm$ 0.002 } &	0.81 $\pm$ 0.002 &	0.803 $\pm$ 0.013 \\
USPS &	0.882 $\pm$ 0.0 &	0.817 $\pm$ 0.003 &	0.824 $\pm$ 0.002 &	0.852 $\pm$ 0.01 &	0.856 $\pm$ 0.001 &	0.818 $\pm$ 0.005 &	{\bf 0.868 $\pm$ 0.014 } \\
Fashion &	0.949 $\pm$ 0.0 &	0.845 $\pm$ 0.003 &	0.856 $\pm$ 0.005 &	0.914 $\pm$ 0.004 &	0.895 $\pm$ 0.002 &	0.85 $\pm$ 0.004 &	{\bf 0.929 $\pm$ 0.004 } \\
\midrule
Avg. rank &	0.0 &	1.15 &	3.05 &	4.45 &	4.75 &	3.05 &	4.55 \\
\bottomrule
\end{tabular}
}
\end{table}

\clearpage

\section{}
\label{App2}
\setcounter{lemma}{0}
\begin{lemma}
\label{maurer}
Let $z_1,\ldots,z_n$ be fixed vectors from $\mathcal{Z}.$ Besides, let $\mathcal{F}$ be a family of $K$-dimensional functions on $\mathcal{Z}.$ We also consider Lipschitz functions $h_i:\mathbb{R}^K \rightarrow \R, i=1,\ldots,n$ with a Lipschitz constant $L>0.$ If there exists $\bar f \in \mathcal{F}$ such that $h_i(\bar f(z_i))=0$ for $i=1,\ldots,n,$ then 
\begin{equation}
\label{Maurer_bound}
\Ex \sup_{f \in \mathcal{F}} \left|\sum_{i=1}^n \varepsilon_i h_i(f(z_i))\right| \leq 2 \sqrt{2} L  \Ex \sup_{f \in \mathcal{F}} \sum_{i=1}^n \sum_{k=1}^K \varepsilon_{i,k} f_k(z_i),
\end{equation}
where $f=(f_1,f_2,\ldots,f_K)$ for each $f \in \mathcal{F}$ and $\{\varepsilon_i\}_i, \{\varepsilon_{i,k}\}_{i,k}$ are independent Rademacher sequences. 
\end{lemma}

\begin{proof}
The key step in the proof is Corollary 1 in \cite{Maurer2016}, which does not require the assumption on existence of $\bar f$  to establish analogue of the lemma above
%and establishes the analog of \eqref{Maurer_bound} 
without absolute values on the left-hand side of \eqref{Maurer_bound}. We show how to strengthen Corollary 1 in \cite{Maurer2016} to \eqref{Maurer_bound}. The price for this improvement will be the additional constant  2   on the right-hand side of the inequality.

For all real $t$ we have $|t| = \max(t,0)+ \max(-t,0).$ Therefore, we obtain
\begin{eqnarray}
\label{Maur1}
&\,&\Ex \sup_{f \in \mathcal{F}} \left|\sum_{i=1}^n \varepsilon_i h_i(f(z_i))\right|=
\Ex \sup_{f \in \mathcal{F}} \left[\max\left(\sum_{i=1}^n \varepsilon_i h_i(f(z_i)),0\right) +\max\left(- \sum_{i=1}^n \varepsilon_i h_i(f(z_i)),0\right)\right]\nonumber \\
&\leq&\Ex \sup_{f \in \mathcal{F}} \max\left(\sum_{i=1}^n \varepsilon_i h_i(f(z_i)),0\right)+
\Ex \sup_{f \in \mathcal{F}} \max\left(\sum_{i=1}^n (-\varepsilon_i) h_i(f(z_i)),0\right).
\end{eqnarray}
Variables $\varepsilon_i$ and $-\varepsilon_i$ have the same distibution, so  \eqref{Maur1} equals
$$
2 \Ex \sup_{f \in \mathcal{F}} \max\left(\sum_{i=1}^n \varepsilon_i h_i(f(z_i)),0\right)=
2\Ex \max\left( \sup_{f \in \mathcal{F}} \sum_{i=1}^n \varepsilon_i h_i(f(z_i)),0\right),
$$
which is $2\Ex  \sup_{f \in \mathcal{F}} \sum_{i=1}^n \varepsilon_i h_i(f(z_i)),$ because
$$\sup_{f \in \mathcal{F}} \sum_{i=1}^n \varepsilon_i h_i(f(z_i)) \geq \sum_{i=1}^n \varepsilon_i h_i(\bar f(z_i))=0.$$ Finally, we apply Corollary 1 in \cite{Maurer2016}.

\end{proof}

\end{document}